\documentclass{article}

\usepackage{microtype}
\usepackage{graphicx}
\usepackage{float}
\usepackage{subfigure}
\usepackage{booktabs} % for professional tables
\usepackage{diagbox}

\usepackage{hyperref}

\usepackage[accepted]{icml2024}

\usepackage{amsmath}
\usepackage{amssymb}
\usepackage{mathtools}
\usepackage{amsthm}
\usepackage{xspace}
\usepackage{multirow}

\usepackage[capitalize,noabbrev]{cleveref}
\usepackage{enumitem}
\setitemize{noitemsep,topsep=0pt,parsep=0pt,partopsep=0pt}

\theoremstyle{plain}
\newtheorem{theorem}{Theorem}[section]

\newtheorem{lemma}[theorem]{Lemma}

\theoremstyle{definition}
\newtheorem{definition}[theorem]{Definition}

\theoremstyle{remark}

\newcommand{\s}{\tilde{\bold{s}}}
\newcommand{\y}{\tilde{\bold{t}}}
\newcommand{\f}{\bold{f}}

\newcommand{\bz}{\bold{h}_t}

\newcommand{\bx}{\bold{h}_s}
\newcommand{\GraphRepr}{\bold{h}_g}
\newcommand{\x}{{h_s}}
\newcommand{\z}{{h_t}}
\newcommand{\wP}{{\bold{P}}}
\newcommand{\DofOut}{c}
\newcommand{\DofHid}{d_h}
\newcommand{\DofRepr}{d_g}
\newcommand{\oram}{O(\lVert\bold{r}-\bold{r}_0\rVert^2)}
\newcommand{\Oram}{O(\lVert\bold{r}-\bold{r}_0\rVert)}

\newcommand{\method}{\textsc{DisGen}\xspace}
\usepackage[textsize=tiny]{todonotes}

\icmltitlerunning{Enhancing Size Generalization in Graph Neural Networks through Disentangled Representation Learning}

\begin{document}

\twocolumn[
\icmltitle{Enhancing Size Generalization in Graph Neural Networks through Disentangled Representation Learning}

% It is OKAY to include author information, even for blind
% submissions: the style file will automatically remove it for you
% unless you've provided the [accepted] option to the icml2024
% package.

% List of affiliations: The first argument should be a (short)
% identifier you will use later to specify author affiliations
% Academic affiliations should list Department, University, City, Region, Country
% Industry affiliations should list Company, City, Region, Country

% You can specify symbols, otherwise they are numbered in order.
% Ideally, you should not use this facility. Affiliations will be numbered
% in order of appearance and this is the preferred way.
\icmlsetsymbol{equal}{*}

\begin{icmlauthorlist}
\icmlauthor{Zheng Huang}{dartmouth}
\icmlauthor{Qihui Yang}{yyy}
\icmlauthor{Dawei Zhou}{vt}
\icmlauthor{Yujun Yan}{dartmouth}

\end{icmlauthorlist}

\icmlaffiliation{dartmouth}{Department of Computer Science, Dartmouth College, Hanover, NH, USA.}
\icmlaffiliation{vt}{Department of Computer Science, Virginia Tech, Blacksburg, VA, USA}
\icmlaffiliation{yyy}{Electrical and Computer Engineering, UCSD, San Diego, USA.}

\icmlcorrespondingauthor{Yujun Yan}{yujun.yan@darmouth.edu}

% You may provide any keywords that you
% find helpful for describing your paper; these are used to populate
% the "keywords" metadata in the PDF but will not be shown in the document
\icmlkeywords{Machine Learning, ICML}

\vskip 0.3in
]

% this must go after the closing bracket ] following \twocolumn[ ...

% This command actually creates the footnote in the first column
% listing the affiliations and the copyright notice.
% The command takes one argument, which is text to display at the start of the footnote.
% The \icmlEqualContribution command is standard text for equal contribution.
% Remove it (just {}) if you do not need this facility.

%\printAffiliationsAndNotice{}  % leave blank if no need to mention equal contribution
\printAffiliationsAndNotice{} % otherwise use the standard text.
\begin{abstract}

%In the past few years,  have become the de facto model of choice for graph classification. 
Although most graph neural networks (GNNs) can operate on graphs of any size, their classification performance often declines on graphs larger than those encountered during training. Existing methods insufficiently address the removal of size information from graph representations, resulting in sub-optimal performance and reliance on backbone models. In response, we propose \method, a novel and model-agnostic framework designed to disentangle size factors from graph representations. \method employs size- and task-invariant augmentations and introduces a decoupling loss that minimizes shared information
% have addressed this challenge through causal models or by accessing to graphs from the test domain. However, a notable correlation exists between learned graph representations from GNNs and graph size, affecting both task performance and model generalizability. 
% we introduce a \yy{novel and model-agnostic} framework, \method, designed to disentangle size factors from graph representations. 
% This is achieved via an augmentation strategy and a novel decoupling loss which minimizes shared information 
in hidden representations, with theoretical guarantees for its effectiveness. Our empirical results show that \method outperforms the state-of-the-art models by up {to 6\% on real-world datasets}, underscoring its
% Extensive experiments on real-world datasets \yy{show that our model 
% demonstrate the framework's 
effectiveness in enhancing the size generalizability of GNNs. Our
codes are available at:  \href{https://github.com/GraphmindDartmouth/DISGEN}{https://github.com/GraphmindDartmouth/DISGEN}.

% This document provides a basic paper template and submission guidelines.
% Abstracts must be a single paragraph, ideally between 4--6 sentences long.
% Gross violations will trigger corrections at the camera-ready phase.
\end{abstract}

\section{Introduction}

% facts: graph representation is related graph size, and this size information bad effect on the test dataset \ref{[xxxx]}. disentange purging size info

% disentangle is used with vae and or splitting other parts, goal, what they did, in our case we utilized disenge to xxxx.

% Graph Neural Networks (GNNs) have emerged as a powerful tool for processing graph data. They have been successfully applied in
% various domains, in various domains, including social networks, citation networks, molecular biology, and chemistry \cite{hamilton2017inductive, chen2022learning, guo2020deep}.

% Graphs in various domains exhibit significant size variations,
Graphs can exhibit significant variations in size, ranging from small molecular structures to large protein networks~\cite{kearnes2016molecular, fout2017protein}, and from concise code snippets with tens of nodes to extensive programs with thousands of nodes \cite{yan2020neural, guo2020graphcodebert}. Training Graph Neural Networks (GNNs) on these datasets is often limited to small graphs due to computational constraints or data availability \cite{zhang2019graph, li2021training}. However, these models are intended for application on larger graphs, necessitating effective size generalization to handle larger test graphs.

Despite the inherent capability of GNNs to process graphs of any size, performance declines are observed when models trained on smaller graphs are applied to substantially larger ones \cite{buffelli2022sizeshiftreg, chen2022learning}. Prior work to address this challenge involves strategies such as accessing test domain graphs~\cite{yehudai2021local}, employing causal modeling for the graph generative process~\cite{chen2022learning, bevilacqua2021size}, and simulating size shifts within the training data through graph coarsening approaches~\cite{buffelli2022sizeshiftreg}. However, these methods inadequately remove
the size information from the graph representations,
resulting in sub-optimal performance and reliance
on backbone models.

Motivated by recent studies \cite{yan2023size, bevilacqua2021size} that reveal the correlation between learned graph representations and their size, we aim to tackle the size generalization problem through disentangled representation learning, a technique known for separating fundamental factors in observed data. However, directly applying existing disentangled representation learning methods \cite{sarhan2020fairness, creager2019flexibly} to our problem poses several challenges. Firstly, these methods typically rely on supervision to disentangle different information, necessitating supervision of graph size in our case. The discrete and unbounded nature of graph sizes, however, complicates their use as direct supervision labels.
Secondly, it remains unclear how to minimize the shared information between size-related and task-related representations with theoretical guarantees. Current practices often use correlation loss \cite{mo2023disentangled} or enforce orthogonality of the representations \cite{sarhan2020fairness} to segregate distinct information. However, there is limited theoretical analysis to substantiate their effectiveness, and in practice, these methods have proven suboptimal for our problem (\Cref{abl}).

To address these challenges, we propose a general \textbf{Dis}entangled representation learning framework for size \textbf{Gen}eralization (\method) of GNNs. To tackle the first challenge, we introduce new augmentation strategies to guide the model in learning relative size information.
Specifically, we create two views—size- and task-invariant
views—for a given input graph and facilitate the learning of
their relative size through a contrastive loss. To tackle the second challenge, we propose a decoupling loss to minimize the shared information between the hidden representations optimized for size- and task-related information, respectively. We further provide theoretical guarantees to justify the effectiveness of the decoupling loss. 
% We conduct experiments showcasing \method's superior performance in size generalization tasks across real-world datasets using various GNN backbones (GCN, GIN and graph transformer).

Our contributions can be summarized as follows:
\begin{itemize}
    \item \textbf{Novel model-agnostic framework:} To the best of our knowledge, \method is the first disentangled representation learning framework to tackle the size generalization problem for GNNs.
    \item \textbf{Novel designs with theoretical guarantees:} We propose new augmentation strategies and novel decoupling loss to segment size- and task-related information. We also provide theoretical guarantees to justify the effectiveness of our proposed loss.
    \item \textbf{Extensive experiments:} Our empirical results show that \method outperforms the state-of-the-art models by up {to 6\% on real-world datasets, highlighting its enhanced size generalizability for GNNs.}
    % underscoring its effectiveness in enhancing the size generalizability of GNNs.   
\end{itemize}
\section{Preliminary}
% \vspace{-0.2cm}

% add graph repre $h_{g_i}$
% d下标有含义
% h_{\mathcal{g}\_i}
$\textbf{Notations. }$ 
Consider a set of $n$ graphs denoted by $\{\mathcal{G}_1, \mathcal{G}_2,..., \mathcal{G}_n\}$, where each $\mathcal{G}_i=(\mathcal{V}_i, \mathcal{E}_i)$ represents the $i$-th graph with $N=|\mathcal{V}_i|$ nodes and $E=|\mathcal{E}_i|$ edges. We denote the neighborhood of a node $v$ as $\mathcal{N}_v$, defined as $\mathcal{N}_v=\{u \mid (u,v) \in \mathcal{E}_i\}$. Furthermore, we denote the size-invariant and task-invariant views augmented from graph $\mathcal{G}_i$ as $\mathcal{G}_i^{(1)}$ and $\mathcal{G}_i^{(2)}$, respectively. As to the matrix representations, 
% Let $\{\mathcal{G}_1, \mathcal{G}_2,..., \mathcal{G}_M \}$ be a set of \yy{M} graphs, \yy{where $\mathcal{G}_i=(\mathcal{V}_i, \mathcal{E}_i)$ is the $i$-th graph with $N=|\mathcal{V}_i|$ nodes and $E=|\mathcal{E}_i|$ edges.}
% , where $M$ is the number of graphs. For a graph $\mathcal{G}_i = (\mathcal{V}_i, \mathcal{E}_i)$, we use $\mathcal{V}_i=\{v_1, v_2,...,v_{N} \}$ to denote the set of nodes of graph $\mathcal{G}_i$ with size $N$, and $\mathcal{E}_i \subseteq \mathcal{V}_i \times \mathcal{V}_i$ to denote the set of edges. Also, 
% We use $\mathcal{N}_v $ to indicate the one-hop neighbor of a specific node $v$. 
we use $\bold{X}_{i} \in \mathbb{R}^{N\times d_f}$ and $\bold{A}_{i} \in \mathbb{R}^{N\times N}$ 
% as $\mathcal{G}_i$'s feature matrix and adjacency matrix,
to represent the feature matrix and adjacency matrix of $\mathcal{G}_i$,
respectively, where $d_f$ is the dimension of node features.
% Graph $\mathcal{g}_i$'s adjacency matrix is defined as $\bold{A}_i \in \mathbb{R}^{N \times N}$, where $\bold{A}_i$
Moreover, 
% we use $\mathcal{G}_i^{(1)}$ and $\mathcal{G}_i^{(2)}$ to denote the 2 augmented views generated from graph $\mathcal{G}_i$. 
we use $\bold{h}_{g\_i}\in \mathbb{R}^{d_g}$ to denote the representation of graph $\mathcal{G}_i$, where $d_g$ is the dimension size. Additional matrix notations will be introduced as the paper progresses. 
Generally, we use a bold lowercase letter to denote a vector and a bold uppercase letter to denote a matrix used in our framework. In addition, superscripts are employed to denote matrices associated with augmented graphs, while subscripts are utilized to indicate matrices specific to a particular graph. We denote the entry at the $(p, q)$ position of matrix $\bold{M}$ as $\bold{M}[p, q]$. Furthermore, for any multivariate function 
% (e.g., $\bold{M}\in\mathbb{R}^{d_2 \times d_3}$, where $d_2$, $d_3$ represent the matrix shape, and $\bold{M}[p,q]$ denotes the entry at $p$-th row and $q$-th column). 
% \zh{Do we assume graph representations and hidden representation share the same size d?}
% To clarify vectors we used, we adopt letter $\DofRepr$, $\DofHid$, $\DofOut$ to represent the length of vectors in graph representation, hidden representation and final output, respectively. 
% Specifically, we use $\GraphRepr\in\mathbb{R}^{\DofHid}$ to denote graph representations obtained by a GNN, $\bold{h}\in\mathbb{R}^{\DofHid}$ to denote hidden representations, 
% % where $d$ is the dimension number, 
% $\bold{h}^{(j)}\in\mathbb{R}^{\DofHid}$, $j=1,2$, to denote hidden representation related to an augmented view, $\bold{t} \in \mathbb{R}^{\DofOut}$, $\bold{s} \in \mathbb{R}^{\DofOut}$ to denote the final output vector containing task related, size related information while $\y$ and $\s$ denoting their corresponding ground truth, and $\bold{M}\in\mathbb{R}^{d_1 \times d_2}$ to denote a matrix.  
% \yqh{For any function} 
$\f(\y,\s): (\mathbb{R}^{\DofOut}, \mathbb{R}^{\DofOut})\mapsto \mathbb{R}^{\DofRepr}$, 
% i.e. $\f(\y,\s)$, 
we use the expression $\partial \f /\partial \s\equiv\bold{0}$
to indicate that for every component function $f_i$ of the multivariate function $\f$, its partial derivative with respect to every component of the input vector $\s$ is always 0, i.e. $\partial f_i/ \partial s_j\equiv0$ for all $i$ and $j$. 
 % \yy{\# of nodes/edges/ adj/label?}

% $G=(V, E)$ be a graph, where V is the set of nodes with size n, and $E \subseteq V \times V$ is the set of edges. For a node $v \in V$, its one-hop neighbor can be represented by $\mathcal{N}_v $. We denote the attributed graph $G$ with a tuple $(\bold{A}^{(G)}, \bold{X}^{(G)})$, where $\bold{A}^{(G)} \in \{0, 1\}^{n \times n}$ is the adjacency matrix (i.e., $\bold{A}^{(G)}_{ij}=1$ iff $e_{ij}^{(G)}=(v_i, v_j)\in E$), and $\bold{X} \in \mathbb{R}^{n \times d}$ is the feature matrix, where $d$ is the feature dimension.

% 小set size < 一个pre defined constant，大 dataset，所有都>一个
% 

% \zhou{Please use the subsection environment.}
\textbf{Problem Setup. }
% \subsubsection{Graph Neural Networks}
% training graphs， size < 一个数， > 一个数，然后trian 在小 （来train一个gnn framework）， test在大
% In this paper, we focus on the graph size generalization task.  Specifically, we are given a set of training data $\{\mathcal{G}_1, \mathcal{G}_2,...,\mathcal{G}_{I1}\}$ that for any $\mathcal{G}_i \in \mathcal{S} $ the size of it is small than $R$, where $R$ is a predefined hyperparameter. For a set of test data $\{\mathcal{G}_1, \mathcal{G}_2,...,\mathcal{G}_{I2}\}$, all the size is large than $R$. Thus, the size generalization task aims to train a GNN on training data that generalize well  
% \yy{Not quite right here as we also use graph transformer as the backbone method. }
In this paper, we define graph size as the number of nodes {in the graph and study the size} generalizability of GNNs \cite{chami2022machine, wu2020comprehensive, maron2018invariant}.
We define a GNN model as size generalizable if it demonstrates generalizability to test graphs with sizes larger than those in the training set.
% exhibiting a different number of nodes from that of the training graphs set. 
We focus on a supervised graph classification task, where each graph $\mathcal{G}_i$ is assigned a label $y_i$, and the goal is to learn a GNN model $f_{\theta}:(\mathbf{A}_{i}, \mathbf{X}_{i}) \mapsto y_i $  that maps each graph $\mathcal{G}_i$ to $y_i$. The graph classification objective is given by a \texttt{CrossEntropy} loss. Our goal is to design a framework, denoted by $g$, where $g\circ f_{\theta}(\mathbf{A}_{i}, \mathbf{X}_{i})$ yields a more accurate estimation of $y_i$ for graphs with sizes larger than those encountered during training.

\textbf{Disentangled Representation Learning. }
% used in representation learning
% add A X
% 定义defintion在 problem defi里，然后用公式在这部分写
% 
Disentangled representation learning aims to separate and isolate the fundamental factors within observed data. By incorporating supervision or prior knowledge, this approach promotes the independence of different factors. This independence can be achieved through various methods, such as optimizing the Pearson’s correlation coefficient to zero \cite{mo2023disentangled}, enforcing orthogonality \cite{sarhan2020fairness}, or minimizing cosine similarity \cite{lidisentangled} of hidden representations. 
% .
% $$
% \mathcal{L}_{cor} = \frac{\text{Cov}(\Phi(\bold{Q}),\Psi \bold{(P)})}{\sqrt{\text{Var}(\Phi(\bold{Q}))} \sqrt{\text{Var}(\Psi(\bold{Q}))}},
% $$
% where $\bold{P}$ and $\bold{Q}$ are the information we want to disentangle, Cov(·, ·) and Var(·) indicate covariance and variance operations, respectively, $\Phi$ and $\Psi$ are measurable functions \cite{gretton2005kernel}. 
% Constraints are then applied to guarantee the high quality of the disentangled information (free from noise or shared content). This is achieved by using reconstruction loss \cite{mo2023disentangled}, 
% enforcing 
% orthogonality \cite{sarhan2020fairness}, or minimizing cosine similarity \cite{lidisentangled}.
% of the corresponding representations \cite{sarhan2020fairness}.

\textbf{Explainable GNN Model. }
Explainable GNN models \cite{ying2019gnnexplainer, luo2020parameterized, li2023interpretable, yan2019groupinn} are effective tools for understanding the predictions made by GNN models. Our focus is on perturbation-based methods \cite{yuan2022explainability, luo2020parameterized}, which assess the importance scores of input graphs by monitoring changes in predictions resulting from different input perturbations. The intuition behind these approaches is that retaining task-relevant information in the inputs should result in predictions similar to the original ones. The explainable model takes as input a graph's adjacency matrix $\bold{A}_i$, feature matrix $\bold{X}_i$, its label $y_i$, and a trained GNN $f_\theta$. The output is a weighted matrix $\bold{M} \in \mathbb{R}^{N \times N}$ that indicates the importance of edges, serving as the explanation for the given label.

\section{Methodology}
% disen to 
% size aware and task aware aug and a decoupling loss
% guide the sep of the two information
% this section includes

% \subsection{Motivation}
% 发现 graph representation 含有size 有关信息 、cite，这些信息有一部分对downstream 是没有帮助的，而且会危害model gene on graph with larger size 。 disen 是一个很好的method可以分离不同的信息，intui 这个disen 可以用来分离size和task 相关的信息来增强gene能力

\begin{figure*}[h]
\vspace{-2cm}

  \centering
  \includegraphics[width=0.9\linewidth]{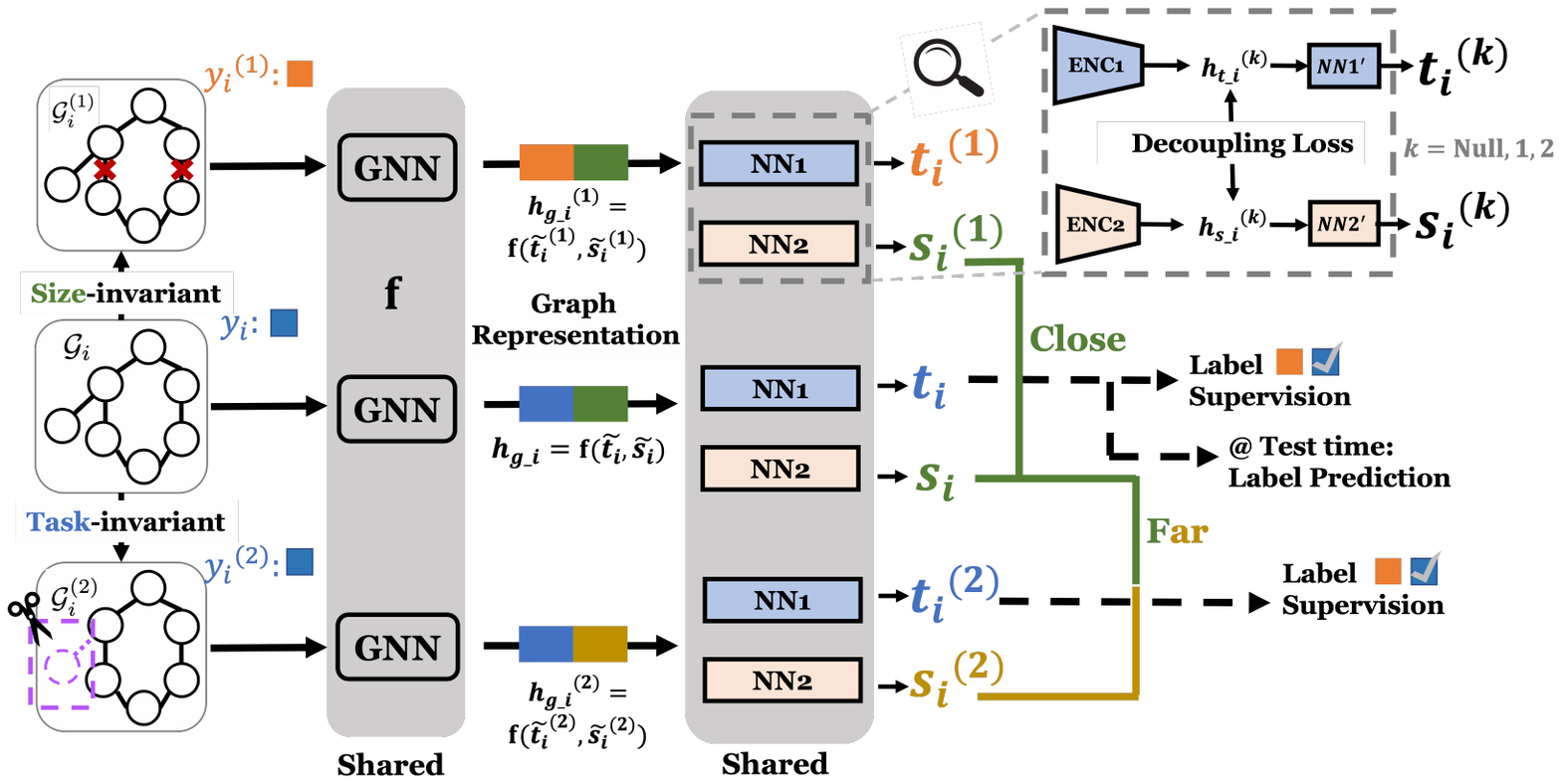}
  \vspace{-2cm}

  \caption{Framework overview: our model augments each graph $\mathcal{G}_i$ with size- and task-invariant views ($\mathcal{G}_i^{(1)}$ and $\mathcal{G}_i^{(2)}$), which, along with the original graph, are processed by a shared GNN backbone. Two encoders then generate size- ($\bold{s}_i$) and task-related ($\bold{t}_i$) representations, respectively. A contrastive loss on size-related representations guides relative size learning, while a decoupling loss ensures the separation of size- and task-related information.}
  \label{fig:framework}
 \vspace{-0.5cm}

\end{figure*}

\begin{figure}[h]
\vspace{-0.8cm}

\label{fig:aug}
  \centering
  \includegraphics[width=\linewidth]{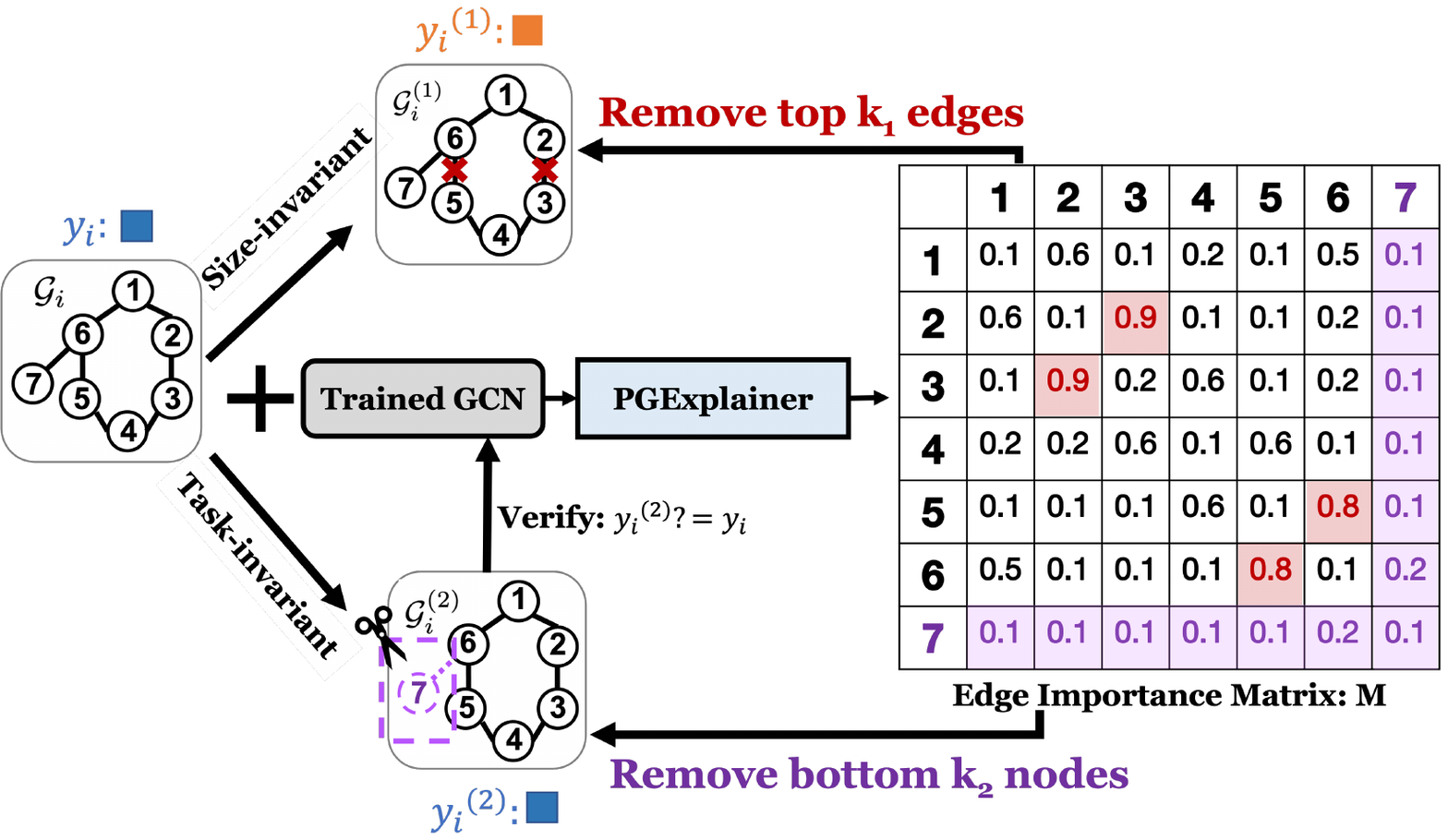}
\vspace{-1cm}

  \caption{Augmentation overview: view $\mathcal{G}_i^{(1)}$ is generated by removing edges that most significantly change the label information, 
  while $\mathcal{G}_i^{(2)}$ results from eliminating nodes that have little impact on the model predictions.}
  % \vspace{-0.2cm}

  \label{fig:augmentation}
\end{figure}

% 以下写在method之后
% 我们的work是based on disen presentation learning
% 写disen presentation learning好处是啥
% 过去的要supervison + 确保独立性
% 过去的work需要直接 supervion
% 我们这边可以完全去除size information在graph representation中，其他work的目标不是完全去除，他们能不能去除是不clear的，但是我们直接去除
% 其他work disentanglement需要 直接提供 和size相关的supervision  和label 的supervsion。（一般已classification的方式完成） 
%dis discrete number
% 所以我们用augmentation，生成两个 task-xxx views
% 进行size 相对大小的比较，metric learning 是一个很好学习相对大小的tools （并不一定提）

% 过去的work是reduce correlation 或者 max indipentence，我们实验发现表现不好，所以我们提出了一个decoupling loss to 减少shared information

% 一般disenanglement learning 会让两个repsentation correlation 是0，或者相互独立， 

% 所以以下我们介绍 1.如何学习size 之间relation， 2.我们如何design decoupling

% \subsection{Motivation}
In this section, we present our framework \method that
%that improves the size generalizability of any GNN backbone.
utilizes disentangled representation learning to distinguish size- and task-related information in graph representations learned by GNN backbones. 
To address the first challenge posed by the discrete and unbounded nature of graph sizes, we propose new augmentation strategies to guide the model in learning relative size information. Specifically, we create two views—size- and task-invariant views—for a given input graph and facilitate the learning of their relative sizes through a contrastive loss.
To address the second challenge of minimizing the shared information to disentangle size- and task-related representations, we introduce a decoupling loss with theoretical guarantees. 
% designed to minimize shared information between hidden representations tailored for size and task specifics.
% to ensure there is no shared information between task-related and size-related information and thus achieve independence.
In this section, we first present an overview of the framework in \Cref{Model Framework}, followed by a discussion on learning size-related information through graph augmentations in \Cref{Augmentation}. Lastly, we introduce the decoupling loss and discuss its theoretical guarantees in \Cref{decoupling model}
% This section is organized as follows: we \yy{first present the overview of our framework in section \ref{Model Framework}, followed by the introduction of learning}
% % then we introduce how to learn the 
% size relative information with graph augmentation in section \ref{Augmentation}.
% % Finally, we demonstrate the decoupling objective function in section \ref{decoupling model}
% \yy{Finally, we introduce the decoupling loss and provide its theoretical guarantees in section \ref{decoupling model}.}

\subsection{Framework Overview}
\label{Model Framework}
As illustrated in Figure \ref{fig:framework}, \textit{\method} initially augments each graph $\mathcal{G}_i$ with two views—size- and task-invariant views ($\mathcal{G}_i^{(1)}$ and $\mathcal{G}_i^{(2)}$). These augmented views, along with the original graph, are then fed into the shared GNN backbone. Subsequently, two encoders are employed to encode the output graph representations into size- and task-related representations ($\bold{s}_i$ and $\bold{t}_i$), respectively. 
%\yy{We assume that $\s$ and $\y \in \mathbb{R}^{\DofOut}$ are unknown ground truth vectors that exclusively encode size-related and task-related information, respectively. The parameter $c$ denotes the minimum dimensionality required for $\s$ and $\y$ to encode this information.}
The task-related representations of both the original and task-invariant graphs ($\bold{t}_i$ and $\bold{t}_i^{(2)}$) are supervised by their shared labels.
A contrastive loss is then applied to the size-related representations of the original graphs and their augmented views ($\bold{s}_i, \bold{s}_i^{(1)} \text{ and } \bold{s}_i^{(2)}$), aiming to guide the model in learning their relative sizes. Additionally, to ensure the separation of size- and task-related information, a decoupling loss is applied to the hidden representations responsible for generating these representations. {In Figure \ref{fig:framework}, the hidden representations for size- and task-related information are denoted by $\bold{h}_{t\_i}^{(k)}$ and $\bold{h}_{s\_i}
^{(k)}$, respectively, where $i$ is the graph index and $k$ specifies different views.}  Finally, during test time, we use the task-related representations $\bold{t}_i$ to predict the label for the graph $\mathcal{G}_i$.

\subsection{Augmentation}
\label{Augmentation}

% 写 explain 一句话，只用score， pertabation based explain gnn model，干了什么事: 给一个graph 输出score。 简单提背后机理

% 写如何得到两个views， 找到一个subgraph 不改变task label
% 得到 score，然后score怎么处理的
% 还要过一遍原来gnn保证label不变

% 先产生两个view，这两个view是啥： 这两个是啥，一个size related 一个task related。我们用这两个view之间的相对距离来学习guide我们的 framework学习size 相关information
% 那么引出两个问题
% 如何创建view
% 如何learn size之间相对关系
In this subsection, we present the augmentation strategies employed to generate the size- and task-invariant views, as shown in Figure \ref{fig:augmentation}. The size-invariant view $\mathcal{G}_i^{(1)}$ maintains the same number of nodes as $\mathcal{G}$ but possesses distinct labels. Conversely, the task-invariant view $\mathcal{G}_i^{(2)}$ shares the same label as $\mathcal{G}$, but differs in the number of nodes. In more detail, view $\mathcal{G}_i^{(1)}$ is generated by selectively removing edges that most significantly change the label information, whereas view $\mathcal{G}_i^{(2)}$ is formed by eliminating nodes that have little impact on the label information. One key challenge in the augmentation process is determining whether the label information has been altered. To address this, {we employ a pre-trained GNN $f_{\theta}$ to monitor the change}. We first pre-train a backbone GNN on the same training graphs, and consider the augmented graph as changing the label information if and only if the predictions from the pre-trained GNN for the augmented graph change. 

To identify critical edges that influence label information and insignificant nodes with minimal impact, we utilize the edge importance matrix derived from graph explainable models \cite{luo2020parameterized, ying2019gnnexplainer}.
% We employ an augmentation strategy to create two views from a given input graph $\mathcal{G}_i$, which guides our framework to learn relative size information. The views include a size-related view $\mathcal{G}_i^{(1)}$, which 
% shares the number of nodes with $\mathcal{G}$ but varies in label information, and a task-related view $\mathcal{G}_i^{(2)}$ that retains the label information of $\mathcal{G}$ but varies in node number. 
% view $\mathcal{G}_i^{(1)}$ is created by removing edges that can change the label information, \yy{while view $\mathcal{G}_i^{(2)}$ is created by removing nodes that do not change the label information.} 
% Perturbation-based graph explanation model is an effective tool for removing such edges. 
In more detail, the explainable model takes a graph $\mathcal{G}_i$, its corresponding label $y_i$, and a trained GNN model $f_{\theta}$, as the inputs, and outputs an edge importance matrix $\bold{M} \in \mathbb{R}^{N\times N}$, where the entry $\bold{M}[i, j]$ indicates the importance score of edge $e_{ij} \in \mathcal{E}_i$. This score measures the influence on the graph label prediction when removing the edge $e_{ij}$—the higher the score, the stronger the influence. As we focus on the undirected graphs, we symmetrize the matrix by: $\hat{\bold{M}} = \frac{1}{2}(\bold{M} +  \bold{M}^{T})$.  To generate the size-invariant view,
% Thus, based on the importance score, 
we remove $k_1$ edges with highest importance scores from the original graph $\mathcal{G}_i$, where $k_1$ is a predefined hyperparameter. To generate the task-invariant view, we compute the node importance scores and remove $k_2$ nodes with the lowest scores from the original graph, along with the edges connecting them,  where $k_2$ is a predefined hyperparameter. We define the node importance scores based on the importance scores of the edges incident to the node. Mathematically, the importance score $m_{v_j}$ of a node $v_j$ is defined as: { $m_{v_j} = \sum_{k \in \mathcal{N}_{v_j}} \hat{\bold{M}}[j,k]$}.
% and get the view's edge set as $\mathcal{E}^{(1)}_i = \mathcal{E}_i \setminus \mathcal{E}_{\text{top-}k_1} $ to create view $\mathcal{G}_i^{(1)} =(\mathcal{V}_i, \mathcal{E}^{(1)}_i)$, where $k_1$ is the predefined hyperparameter.
% Secondly, to generate view $\mathcal{G}_i^{(2)}$, we propose to calculate the node importance score. Similar to the edge importance score, the node score indicates the influence w.r.t. the graph label prediction if removing a node $v_j \in \mathcal{V}_i$, the higher the score, the stronger the influence of a node. Based on the weighted mask matrix $\bold{M}$, this score can be calculated by aggregating all the scores of the edges that connect to a node $v_j$: 
% As a result, we remove the least-$k_2$ nodes and the edges connect to them based on the importance score and generate view $\mathcal{G}_i^{(2)} = (\mathcal{V}^{(2)}_i, \mathcal{E}^{(2)}_i)$, where 
% $\mathcal{V}^{(2)}_i $and $\mathcal{E}^{(2)}_i$ are the node set and edge set correspondingly and $k_2$ is the predefined hyperparameter. 
To verify that the task-invariant view $\mathcal{G}_i^{(2)}$ maintains the same label as the original graph, we feed it to the pre-trained GNN $f_{\theta}$, and test if the following holds:
% After that, we verify if view $\mathcal{G}_i^{(2)}$ shares the same label with $\mathcal{G}$ by \zh{passing them to the trained GNN} $f_{\theta}$:
$
f_{\theta}(\mathcal{G}) = f_{\theta}(\mathcal{G}_i^{(2)})
$. 
If the condition is not met, we reduce $k_2$ to $0.9 k_2$ and repeat the above process.

The views $\mathcal{G}_i^{(1)}$ and $\mathcal{G}_i^{(2)}$ are then utilized to guide our framework to learn the relative size information. 
% Our reasoning is that graph sizes cannot be directly utilized as a supervision signal due to their discrete nature. 
Inspired by contrastive learning \cite{mo2023disentangled, li2021disentangled}, we modify the contrastive loss to convey to the model that the size-invariant view maintains the same size as the original graph while the task-invariant view has a different size. Specifically, we aim for the learned size representations of $\mathcal{G}_i$ and $\mathcal{G}_i^{(1)}$ ($\bold{s}_i \text{ and } \bold{s}_i^{(1)}$) to be close to each other as the two views share the same graph size. Conversely, we expect the size representations of $\mathcal{G}_i$ and $\mathcal{G}_i^{(2)}$ ($\bold{s}_i \text{ and } \bold{s}_i^{(2)}$) to be far away from each other, reflecting their distinct graph sizes. If we denote $c_1$ ($c_2$) as the cosine similarity between $\bold{s}_i$ and $\bold{s}_i^{(1)}$ ($\bold{s}_i^{(2)}$), the contrastive loss is given by:
% we propose to learn relative size information based on two views $\mathcal{G}_i^{(1)}$, $\mathcal{G}_i^{(2)}$ and the given graph $\mathcal{G}_i$. To be more specific, 
% let $\bold{s}_i, \bold{s}_i^{(1)}, \bold{s}_i^{(2)}$ denote the learnable size information representations of $\mathcal{G}_i, \mathcal{G}_i^{(1)}$ and $\mathcal{G}_i^{(2)}$ from our framework, respectively. 
% On one hand, we ensure the size representations of $\mathcal{G}_i$ and $\mathcal{G}_i^{(1)}$ to be close to each other as the two views share the same graph size, the distance is denoted as $d_1$. On the other hand, as $\mathcal{G}_i$ and $\mathcal{G}_i^{(2)}$ vary in size, we encourage the distance $d_2$ between them to far from each other. As mentioned above, we define the first objective function between views: 
$$
\mathcal{L}_s = -log \left( \frac{exp(c_1)/\tau}{exp(c_1)/\tau + exp(c_2)/\tau} \right), 
$$
where $\tau$ is a hyperparameter. 

% Then, let $\bold{t}_i, \bold{t}_i^{(1)}, \bold{t}_i^{(2)}$ denotes the learned task information representations of $\mathcal{G}_i, \mathcal{G}_i^{(1)}$ and $\mathcal{G}_i^{(2)}$ from our framework. 
In addition to the contrastive loss, we also provide supervision on the task-related representations $\bold{t}_i$ and $\bold{t}_i^{(2)}$ to encourage the learning of task information, where $\bold{t}_i$ and $\bold{t}_i^{(2)}$ share the same label. The supervision loss is given by:
% Aiming to ensure these representations only contain the task-related information we discard $\bold{t}_i^{(1)}$ and employ task label $\bold{y}$ as supervision only on $\bold{t}_i$ and $\bold{t}_i^{(2)}$ since they share the same label. 
% As a result, the supervision of task-related information is achieved with the cross-entropy loss:
$$
% \mathcal{L}_t = -\alpha_1 \sum_{c=1}^{C}\bold{y} \log(\bold{t}_i) - \alpha_2 \sum_{c=1}^{C} \bold{y} \log(\bold{t}_i^{(2)}),
\mathcal{L}_t = -\alpha_1\texttt{CE}(y_i, \bold{t}_i)- \alpha_2\texttt{CE}(y_i, \bold{t}_i^{(2)}),
$$
% \yy{wrong formula}
% where $C$ is the number of classes, 
where \texttt{CE} denotes the \texttt{CrossEntropy} loss, $\alpha_1$ and $\alpha_2$ are hyperparameters for regularization.

\subsection{Decoupling Loss}
\label{decoupling model}
\subsubsection{Design}
% Our decoupling model first decouples the size hidden representation and task hidden representation from a learned graph representation. After this, we propose a decoupling loss to ensure there is no shared information between the two hidden representations. 
In this subsection, we introduce the decoupling loss, which aims to minimize the shared information between the hidden representations optimized for size- and task-related information. We denote the size- and task-related hidden representations of a graph as $\bold{h_s}\in \mathbb{R}^{d_h}$ and $\bold{h_t} \in \mathbb{R}^{d_h}$, respectively. The size-related hidden representations of all original and augmented graphs in a batch of size $b$ are stacked into a matrix $\bold{H}_s =\begin{bmatrix} \bold{h}_{s\_1}, \bold{h}_{s\_1}^{(1)}, \bold{h}_{s\_1}^{(2)},...,  \bold{h}_{s\_b}, \bold{h}_{s\_b}^{(1)}, \bold{h}_{s\_b}^{(2)} \end{bmatrix}^T $. Similarly, the matrix of stacked task-related hidden representations in a batch is denoted by $\bold{H}_t =\begin{bmatrix} \bold{h}_{t\_1}, \bold{h}_{t\_1}^{(1)}, \bold{h}_{t\_1}^{(2)},...,  \bold{h}_{t\_b}, \bold{h}_{t\_b}^{(1)}, \bold{h}_{t\_b}^{(2)} \end{bmatrix}^T $.
% Firstly, following the disentanglement learning methods, to decouple the size hidden representation $\bold{h_s}\in \mathbb{R}^{d_1}$ and task  hidden representation $\bold{h_t} \in \mathbb{R}^{d_1}$ from a graph representation $\bold{h_g} \in \mathbb{R}^{d_2}$ (learned from a GNN backbone), we employ two encoders with unshared parameters to map the graph representation $\bold{h}_g$ into task-related representation $\bold{h}_t$ and size-related representation $\bold{h}_s$
% Secondly, we use a decoupling loss to enforce no shared information between size-related hidden representation and task-related hidden representation. 
Our rationale is that if $\bold{H}_s$ and $\bold{H}_t$ contain the same information, there exists a function $f_p \in \mathcal{F}$ that can transform one representation to another: {$\bold{H}_s = f_p(\bold{H}_t)$}. When we restrict the function class $\mathcal{F}$ to linear projection functions, the substantial similarity in information between $\bold{H}_s$ and $\bold{H}_t$ suggests the existence of a linear projection $\bold{P}$ such that the difference $\bold{H}_t \bold{P} - \bold{H}_s$ is sufficiently small. For any given $\bold{H}_s$ and $\bold{H}_t$, we can find an optimal projection plane $\bold{P}_{opt}$ that minimizes the mapping residual from {$\bold{H}_t$ to $\bold{H}_s$}:
% Under the constraint of linear projection function, the projection can be written as $\bold{H}_t = \bold{P} \bold{H}_s$, where $\bold{P}$ is a projection matrix. Thus, if $\bold{H}_s$ and $\bold{H}_t$, have enough shared information, we can find a projection that minimizes the Frobenius norm of the projection:
\begin{equation}
    \bold{P}_{opt} = \min_\bold{P} \Bigg\Vert \bold{H}_t  \bold{P} - \bold{H}_s \Bigg\Vert_F,
\label{eq1}
\end{equation}
where $\Vert\cdot\Vert_F$ denotes the Frobenius norm of a matrix. Let $D$ denote the residual under optimal linear projection: {$D=\Vert \bold{H}_t   \bold{P}_{opt} - \bold{H}_s \Vert_F$}. We can use {$D$} to quantify the shared information between $\bold{H}_s$ and $\bold{H}_t$. A small value of $D$ indicates a substantial overlap in information between $\bold{H}_s$ and $\bold{H}_t$, while a large {$D$} suggests minimal shared information. Recall that our goal is to minimize the shared information, thus we aim to train our framework such that $\bold{H}_s$ and $\bold{H}_t$ satisfy: 
% we denote the value as $r$. Based on the rationale, we  maximize $r$ to ensure there is no shared information between size-related hidden representation and task-related hidden representation. Formally, it can be denoted as:
\begin{equation}
    \max_{\bold{H}_t, \bold{H}_s} \min_\bold{P} \Bigg\Vert \bold{H}_t  \bold{P} - \bold{H}_s \Bigg\Vert_F,
\label{eq2}
\end{equation} 
% where $\bold{H}_t =\begin{bmatrix} \bold{h}_{t1} \_ \bold{h}_{t1}^{(1)} \_\bold{h}_{t1}^{(2)},...,  \bold{h}_{tb} \bold{h}_{tb}^{(1)} \bold{h}_{tb}^{(2)} \end{bmatrix}^T $ stacks a batch $b$ of all task-related hidden representation, and $\bold{H}_s$ can be created with the same operation. 
The best linear projection plane $\bold{P}_{opt}$ can be obtained by setting the derivative of \Cref{eq1} with respect to $\bold{P}$ to zero. {Thus $\bold{P}_{opt}$ is given by:}
{\begin{equation}
    \bold{P}_{opt} = (\bold{H}_t^T\bold{H}_t)^{-1}\bold{H}_t^T\bold{H}_s.
\label{e3}
\end{equation}}
% \Cref{eq2} reduces to:
% \begin{equation}
%     \max_{\bold{H}_t, \bold{H}_s}\Bigg\Vert \bold{H}_t (\bold{H}_t^T\bold{H}_t)^{-1}\bold{H}_s^T\bold{H}_t - \bold{H}_s \Bigg\Vert_F,
% \label{e3}
% \end{equation}
As a result, we denote our decoupling loss as:
\begin{equation}
\mathcal{L}_d = \frac{1}{D^2}, \; \; \; D=\Bigg\Vert \bold{H}_t   \bold{P}_{opt} - \bold{H}_s \Bigg\Vert_F.
\label{e4}
\end{equation}
While the computation of $\bold{P}_{opt}$ requires inversion, it is noteworthy that in practice, the matrix size for inversion is small, resulting in negligible slowdowns in training time.

After decoupling $\bold{h}_t$ and $\bold{h}_s$, a prediction $\bold{t}$ is made by a neural network, as well as $\bold{s}$. As a result, the objective function of our framework is: 
$
\mathcal{L} = \beta_1 \mathcal{L}_s + \beta_2 \mathcal{L}_t + \beta_3 \mathcal{L}_d,
$ where $\beta_1, \beta_2$ and $\beta_3$ are weights of different loss functions.

\subsubsection{Theoretical Analysis}
In this subsection, we provide a theoretical analysis to show the effectiveness of the decoupling loss. We begin with the assumptions and definitions. 
% The definitions for $\s$ and $\y$ can be found in \Cref{Model Framework}.
We assume that $\s$ and $\y \in \mathbb{R}^{\DofOut}$ are unknown ground truth vectors for graph $\mathcal{G}$. The parameter $c$ denotes the minimum dimensionality required for $\s$ and $\y$ to exclusively encode size-related and task-related information, respectively.
% $\s$ ($\y$) is a vector with minimum length $c$ which sorely encode the size-related information
% the unknown ground truth for $\bold{s}$ ($\bold{t}$), which only contains size-related information (task-related) information.
We further assume that
the graph representation $\GraphRepr$ is a function of 
% the output vector obtained by our model, 
$\s$ and $\y$,
i.e. $\GraphRepr=\f(\y,\s)$. Next, we formally define decoupling as follows:
% \yy{Next we show that }maximizing the decoupling loss helps the encoder functions \textbf{ENC}$_i(\cdot)$ to 
% \yy{exclusively encode one type of the information (size- or task-related information)}
% % compress the information of one the the two aspects 
% to the hidden representations \yy{$\bold{h_s}$ and $\bold{h_t}$}. Finally, neural networks are used to approximate a function to extract the remained information in the hidden representation and give a final disentangled representation output.

% To address it mathematically, \yy{we begin by formally defining the concept of decoupling.}
% defining what is decoupling is essential. We begin by introducing the definitions.

% \yy{Qihui, please move the following notations to Section 2.} Done.

\begin{definition}
  We say that $\y$ and $\s$ can be decoupled from $\f$ by $\bold{g}_1$ and $\bold{g}_2$,
  if given
  % $\forall$ 
  $\f(\y,\s): (\mathbb{R}^{\DofOut}, \mathbb{R}^{\DofOut})\mapsto \mathbb{R}^{\DofRepr}$, $\exists$ $\bold{g}_1, \bold{g}_2:\mathbb{R}^{\DofRepr} \mapsto \mathbb{R}^{\DofHid}$, satisfying that
  $$\frac{ \partial \bold{g}_1(\f(\y,\s))}{ \partial \s} \equiv \bold{0},$$
  and
  $$\frac{ \partial \bold{g}_2(\f(\y,\s))}{ \partial \y }\equiv\bold{0}.$$
\end{definition}
Note that if $\y$ and $\s$ can be decoupled from $\f$, then the following equations hold:
 % the resulting composite function
 $$\bold{g}_1(\f(\y,\s)) = \bold{w}_1(\y),$$
$$\bold{g}_2(\f(\y,\s)) = \bold{w}_2(\s),$$
 for some functions $\bold{w}_1$ and $\bold{w}_2$.

Our goal is to show the connection between
{maximizing $D$ (\cref{e4})}
% maximizing \cref{e3} 
and decoupling $\y$ and $\s$ from $\f$. Before presenting the rationale, we first show that
% {it is essential to address some challenges in the 
% Before giving the direct rationale between decoupling loss and the actual decoupling results, 
 % The maximizing condition \yy{in \cref{e3}} is hard to transfer into a applicable mathematical condition.
 % \yy{The maximizing condition in \cref{e3} is challenging to translate into a practical mathematical expression.}
 % However, the opposite condition of maximizing, minimizing, can be approximated by an equation \yy{under specific conditions.}
 % % if some conditions are met. 
 % Therefore, we start from proving some properties of minimizing $D$ in \yy{\cref{e4}}, \yy{ultimately providing a rationale for its maximization.}
% Referring to the functional structure of the model, now we prove that the structure can be guaranteed if we maximize the decoupling loss. 
% But before that, we need to prove a 
% This theorem makes sure that the graph representation function can not by decoupled by the \rm{\textbf{ENC}} function if minimizing the decoupling loss.
% \yy{The following theorem proves that }
$\y$ and $\s$ cannot be decoupled from $\f$ if $D$ is sufficiently small.
For simplicity, we may express the composite function $\textbf{\rm{ENC}}_i\circ\f(\cdot,\cdot)$ as a function of $\bold{r}\in \mathbb{R}^{2\DofOut}$, i.e., $\textbf{\rm{ENC}}_i\circ\f (\cdot): \mathbb{R}^{2\DofOut} \mapsto \mathbb{R}^{\DofHid}$, where $\bold{r}$ is obtained by stacking $\y$ and $\s$.
% \yqh{we may use another expression of the composite function $\textbf{\rm{ENC}}_i\circ\f(\cdot,\cdot)$. We stack $\y$ and $\s$ to get a vector $\bold{r}\in \mathbb{R}^{2\DofOut}$, i.e. $\bold{r}=[\y^\top\ \s^\top]^\top$. The composite function then can be re-written into this form: $\textbf{\rm{ENC}}_i\circ\f: \mathbb{R}^{2\DofOut} \mapsto \mathbb{R}^{\DofHid}$.}

\begin{theorem}\label{thm1.1}
% Assume that the composite functions $\textbf{\rm{ENC}}_i\circ\f (\cdot)$, $i \in \{1,2\}$, defined on \yy{a closed set $S\in 
% \mathbb{R}^{2\DofOut}$ are twice differentiable at a point $\bold{r}_0$, and their gradient matrices $\nabla\bz(\nabla\bx)$ are full rank,}
% % an open set $S\in 
% % \mathbb{R}^{2\DofOut}$ are twice 
% % differentiable on $S$, 
% and that the dimension $\DofOut$ of $\y$ ($\s$) and the dimension $\DofHid$ of $\bz$ ($\bx$) satisfy: $\DofHid \geq 2\DofOut+1$,
% \yy{Consider the composite functions $\textbf{\rm{ENC}}_i \circ \f (\cdot)$, $i \in {1,2}$, defined on a closed set $S \in \mathbb{R}^{2\DofOut}$. Assume that these composite functions are twice differentiable at some point $\bold{r}_0$, and their gradients $\nabla\bz$ and $\nabla\bx$ at $\bold{r}_0$ are nonzero matrices. Furthermore, assume that the dimension $\DofOut$ of $\y$ ($\s$) and the dimension $\DofHid$ of $\bz$ ($\bx$) satisfy the condition: $\DofHid \geq 2\DofOut + 1$.}
{Consider the composite functions $\textbf{\rm{ENC}}_i \circ \f (\cdot)$, $i \in {1,2}$, defined on a closed set $S \in \mathbb{R}^{2\DofOut}$. Assume that these composite functions are twice differentiable at some point $\bold{r}_0$, and the gradients $\nabla\bz$ and $\nabla\bx$ at $\bold{r}_0$ are nonzero matrices. Furthermore, assume that the dimension $\DofOut$ of $\y$ ($\s$) and the dimension $\DofHid$ of $\bz$ ($\bx$) satisfy the condition: $\DofHid \geq 2\DofOut + 1$. } Then $\forall \bold{r} \in S$, $\exists \wP$ {of full rank}, and {some constants $\iota_1,\iota_2$}, {such that:}
\begin{equation}
\label{eq:zpx}
\resizebox{\hsize}{!}{$
\bz(\bold{r})\wP=\bx(\bold{r})+O(\lVert\bold{r}-\bold{r}_0\rVert^2), \; \iota_1<\lim_{\bold{r}\rightarrow\bold{r}_0}\lVert\frac{O(\lVert\bold{r}-\bold{r}_0\rVert^2)}{\lVert\bold{r}-\bold{r}_0\rVert^2}\rVert<\iota_2$}
  % \bZ\wP=\bX+o(|\bold{r}-\bold{r}_0|^1)
  \vspace{0.1cm}
\end{equation}
$\Rightarrow$ $\y$ and $\s$ can not be decoupled from $\f(\cdot,\cdot)$ by the \rm{\textbf{ENC}}$_i(\cdot)$.
% \vspace{-0.5cm}
\begin{proof}
The hidden representations $\bz$ and $\bx$ are outputs of the encoders, which are functions of $\bold{r}$:
% Let
% $$
% \bX =
%   \begin{bmatrix}
%     \x_{\_11} & \cdots & \x_{\_1\DofHid}\\
%     \vdots & {} & \vdots \\
%     \x_{\_b1} & \cdots & \x_{\_b\DofHid}
%   \end{bmatrix},\
%   \bZ =
%   \begin{bmatrix}
%     \z_{\_11} & \cdots & \z_{\_1\DofHid}\\
%     \vdots & {} & \vdots \\
%     \z_{\_b1} & \cdots & \z_{\_b\DofHid}
%   \end{bmatrix}
% $$
% where $b$ and $\DofHid$ stand for the batch size and the length of the vector $\bx$ ($\bz$), respectively. 
% \yqh{We assume that each row of $\bX$ and $\bZ$ varies only because of the change in $\y$ and $\s$. Without loss of generality, we can re-written the two matrix into two functions, }
% \vspace{-0.1cm}
\begin{equation}
\begin{aligned}
    % \yqh{
    % \bZ(\bold{r})=&\bZ(\y,\s)=\textbf{ENC}_1\circ \f,\\
    % \bX(\bold{r})=&\bX(\y,\s)=\textbf{ENC}_2\circ \f.}
    \bz(\bold{r})=&\bz(\y,\s)=\textbf{ENC}_1\circ \f(\bold{r}),\\ \bx(\bold{r})=&\bx(\y,\s)=\textbf{ENC}_2\circ \f(\bold{r}).
\end{aligned}
% \vspace{-0.1cm}
\end{equation}
% \yqh{related with $\y$ and $\s$. 
% % We can view every row of $\bX$ ($\bZ$) as a group of realization of a group of functions since every row is only determined by the graph representation, hence $\y$ and $\s$.
% We can further expand the functions,}
We further expand each composite function into a set of functions 
%$h_{t\_i}(\y,\s)$ % 
$\z_{\_i}(\y,\s)$ 
and $\x_{\_i}(\y,\s): \mathbb{R}^{2\DofOut} \mapsto \mathbb{R}$
% \vspace{-0.1cm}
\begin{equation}
\begin{aligned}
  \bz(\y,\s) &=
  \begin{bmatrix}
    \z_{\_1}(\y,\s) & \cdots & \z_{\_\DofHid}(\y,\s)
  \end{bmatrix}\\
  &=\begin{bmatrix}
    \z_{\_1}(\bold{r}) & \cdots & \z_{\_\DofHid}(\bold{r})
  \end{bmatrix}=\bz(\bold{r}),\\
  \bx(\y,\s) &=
  \begin{bmatrix}
    \x_{\_1}(\y,\s) & \cdots & \x_{\_\DofHid}(\y,\s)
  \end{bmatrix}\\
  &=\begin{bmatrix}
    \x_{\_1}(\bold{r}) & \cdots & \x_{\_\DofHid}(\bold{r})
  \end{bmatrix}=\bx(\bold{r}).
 % \yqh{
 %  \bZ(\y,\s) &=
 %  \begin{bmatrix}
 %    \z_{\_1}(\y,\s) & \cdots & \z_{\_\DofHid}(\y,\s)
 %  \end{bmatrix}\\
 %  &=\begin{bmatrix}
 %    \z_{\_1}(\bold{r}) & \cdots & \z_{\_\DofHid}(\bold{r})
 %  \end{bmatrix}\\
 %  &=\bZ(\bold{r}),\\
 %  \bX(\y,\s) &=
 %  \begin{bmatrix}
 %    \x_{\_1}(\y,\s) & \cdots & \x_{\_\DofHid}(\y,\s)
 %  \end{bmatrix}\\
 %  &=\begin{bmatrix}
 %    \x_{\_1}(\bold{r}) & \cdots & \z_{\_\DofHid}(\bold{r})
 %  \end{bmatrix}\\
 %  &=\bX(\bold{r}).}
  \end{aligned}
  % \vspace{-0.1cm}
\end{equation}
We give the proof by contradiction.

We assume that $\y$ and $\s$ can be decoupled from $\f$ by the \rm{\textbf{ENC}}$_i(\cdot)$ function, then:
% \vspace{-0.1cm}
\begin{equation}
\begin{aligned}
  % \bX(\s) &=
  % \begin{bmatrix}
  %   \x _{\_ 1}(\s) & \cdots & \x _{\_ \DofHid}(\s)
  % \end{bmatrix},\\
  % \bZ(\y) &=
  % \begin{bmatrix}
  %   \z_{\_ 1}(\y) & \cdots & \z_{\_ \DofHid}(\y)
  % \end{bmatrix}
  \bx(\s) &=
  \begin{bmatrix}
    \x _{\_ 1}(\s) & \cdots & \x _{\_ \DofHid}(\s)
  \end{bmatrix},\\
  \bz(\y) &=
  \begin{bmatrix}
    \z_{\_ 1}(\y) & \cdots & \z_{\_ \DofHid}(\y)
  \end{bmatrix}.
  \end{aligned}
  % \vspace{-0.1cm}
\end{equation}
% \vspace{-0.1cm}
Using \Cref{eq:zpx}, we obtain the following expression:
\begin{equation}
  \begin{bmatrix}
    \z_{\_ 1}(\y) & \cdots & \z_{\_ \DofHid}(\y)
  \end{bmatrix}\wP_i=\x_{\_ i}(\s)+\oram,
\end{equation}
where $\wP_i$ stands for the $i$-th column of $\wP$. Taking the partial derivatives with respect to every component $\tilde{s}_j$ of $\s$ on both sides of the equation, we get:
% \vspace{-0.2cm}
\begin{equation}\label{eq:solutionx}
\resizebox{\columnwidth}{!}{$
\begin{aligned}
  \frac{\partial \x_{\_ i} (\s) }{\partial \tilde{s}_j}&=
  \begin{bmatrix}
    \frac{\partial \z_{\_1}(\y)}{\partial \tilde{s}_j} & \cdots & \frac{\partial \z_{\_\DofHid}(\y)}{\partial \tilde{s}_j}
  \end{bmatrix}\wP_i-\Oram\\
  &={\begin{bmatrix}
    0 & \cdots & 0
  \end{bmatrix}\wP_i}-\Oram.
\end{aligned}
$}
  % \vspace{-0.1cm}
\end{equation}
When $\bold{r} \rightarrow \bold{r}_0$, $\Oram\rightarrow \bold{0}.$
% , \yy{$\frac{\partial \z_{\_ i} (\y)}{\partial \tilde{s}_j}\rightarrow 0.$} 
Thus:
% \vspace{-0.2cm}
\begin{equation}
    \frac{\partial \x_{\_ i} (\s) }{\partial \tilde{s}_j}|_{\bold{r}_0}=0, \; \text{where } i\in \{1,\ldots,\DofHid\},  j\in \{1,\ldots,c\}.
    \label{eq_11}
\end{equation}
Similarly, taking the partial derivatives w.r.t. any $\tilde{t}_j$ on both sides of the equation yields:
% \vspace{-0.1cm}
{\footnotesize
\begin{equation}
\resizebox{\columnwidth}{!}{% 
    $
\begin{aligned}
  &\begin{bmatrix}
  \frac{\partial \x_{\_1} (\s) }{\partial \tilde{t}_1} & \cdots & \frac{\partial \x_{\_ \DofHid} (\s) }{\partial \tilde{t}_1}\\
  \cdots & \cdots & \cdots\\
    \frac{\partial \x_{\_1} (\s) }{\partial \tilde{t}_c} & \cdots & \frac{\partial \x_{\_ \DofHid} (\s) }{\partial \tilde{t}_c}
  \end{bmatrix}
  =
  \begin{bmatrix}
  \frac{\partial \z_{\_1}(\y)}{\partial \tilde{t}_1} & \cdots & \frac{\partial \z_{\_\DofHid}(\y)}{\partial \tilde{t}_1}\\
  \cdots & \cdots & \cdots\\
    \frac{\partial \z_{\_1}(\y)}{\partial \tilde{t}_c} & \cdots & \frac{\partial \z_{\_\DofHid}(\y)}{\partial \tilde{t}_c}
  \end{bmatrix}\wP \\& -\Oram = \bold{0}.
  % \begin{bmatrix}
  %   \yqh{\oram}& \cdots & \yqh{\oram}
  % \end{bmatrix}^\top\\
  % \equiv&\bold{0}^\top.
  \end{aligned}$}
\end{equation}}
% Inherently, since $\wP$ is obtained by matrix with full rank, it also has full rank. 
Since $\wP$ is a matrix of full rank, when $\bold{r} \rightarrow \bold{r}_0$, $\Oram\rightarrow \bold{0}$.
% , \yy{$\frac{\partial \x_{\_i}(\s)}{\partial t_j}\rightarrow 0$}, where $i\in \{1,\ldots,\DofHid\}$, 
Thus we have:
% $$\begin{bmatrix}
%     \frac{\partial \z_{\_ 1}(\y)}{\partial t_j} & \cdots & \frac{\partial \z_{\_ \DofHid}(\y)}{\partial t_j}
%   \end{bmatrix}\wP
%   =\bold{0}$$ 
  % as a linear equation, then we can find that 
 % \yy{Since $\wP$ is a matrix with full rank, } the only solution is:
  \begin{equation}\label{eq:solutionz}
    \frac{\partial \z_{\_i}(\y)}{\partial t_j}|_{\bold{r}_0}=0, \; \text{where } i \in \{1,\ldots,\DofHid\}, j\in \{1,\ldots,c\}
  \end{equation}
Considering \Cref{eq_11} and \Cref{eq:solutionz}, 
% and noting that $\bold{r}_0$ can be any point in $S$, 
it can be concluded that {$\nabla\bz$ and $\nabla\bx$ at $\bold{r}_0$ are zero matrices.}
% all functions must be constant,
% rendering them invalid and 
This result contradicts our initial assumptions.
% From (\ref{eq_11}) and (\ref{eq:solutionz}) 
% we find that all functions are near constant locally, which \yy{are not valid and}
% contradicts our assumptions. 
% on their validity. 
Thus, $\y$ and $\s$ can not be decoupled from $\f$ by the \rm{\textbf{ENC}}$_i(\cdot)$ functions.
\end{proof}
\end{theorem}
Theorem~\ref{thm1.1} provides a sufficient condition for $\y$ and $\s$ not to be decoupled from $\f$ by \rm{\textbf{ENC}}$_i(\cdot)$. Next, we will present a necessary condition with an additional constraint, beyond those specified in Theorem~\ref{thm1.1}.We begin by defining two matrices at $\boldsymbol{r}_0$:
{$$\resizebox{0.6\columnwidth}{!}{ 
    $\bold{B}=\begin{bmatrix}
\x_{\_1}(\boldsymbol{r}_0) & \cdots & \x_{\_\DofHid}(\boldsymbol{r}_0)\\
\frac{\partial\x_{\_1}(\boldsymbol{r}_0)}{\partial \tilde{t}_1} & \cdots & \frac{\partial\x_{\_\DofHid}(\boldsymbol{r}_0)}{\partial \tilde{t}_1}\\
\cdots & \ & \cdots\\
\frac{\partial\x_{\_1}(\boldsymbol{r}_0)}{\partial \tilde{t}_c} & \cdots & \frac{\partial\x_{\_\DofHid}(\boldsymbol{r}_0)}{\partial \tilde{t}_c}\\
\frac{\partial\x_{\_1}(\boldsymbol{r}_0)}{\partial \tilde{s}_1} & \cdots & \frac{\partial\x_{\_\DofHid}(\boldsymbol{r}_0)}{\partial \tilde{s}_1}\\
\cdots & \ & \cdots\\
\frac{\partial\x_{\_1}(\boldsymbol{r}_0)}{\partial \tilde{s}_c} & \cdots & \frac{\partial\x_{\_\DofHid}(\boldsymbol{r}_0)}{\partial \tilde{s}_c}
\end{bmatrix},$}$$}
{$$\resizebox{0.6\columnwidth}{!}{% 
    $\bold{C}=\begin{bmatrix}
\z_{\_1}(\boldsymbol{r}_0) & \cdots & \z_{\_\DofHid}(\boldsymbol{r}_0)\\
\frac{\partial\z_{\_1}(\boldsymbol{r}_0)}{\partial \tilde{t}_1} & \cdots & \frac{\partial\z_{\_\DofHid}(\boldsymbol{r}_0)}{\partial \tilde{t}_1}\\
\cdots & \ & \cdots\\
\frac{\partial\z_{\_1}(\boldsymbol{r}_0)}{\partial \tilde{t}_c} & \cdots & \frac{\partial\z_{\_\DofHid}(\boldsymbol{r}_0)}{\partial \tilde{t}_c}\\
\frac{\partial\z_{\_1}(\boldsymbol{r}_0)}{\partial \tilde{s}_1} & \cdots & \frac{\partial\z_{\_\DofHid}(\boldsymbol{r}_0)}{\partial \tilde{s}_1}\\
\cdots & \ & \cdots\\
\frac{\partial\z_{\_1}(\boldsymbol{r}_0)}{\partial \tilde{s}_c} & \cdots & \frac{\partial\z_{\_\DofHid}(\boldsymbol{r}_0)}{\partial \tilde{s}_c}
\end{bmatrix}.$}$$}
\vspace{0.1cm}
% \yy{
\begin{theorem}\label{thm1.2}
Given the conditions in \Cref{thm1.1}, we assume that the matrices $\boldsymbol{B}$ and $\boldsymbol{C}$, defined at $\bold{r}_0$ satisfy that $\exists \boldsymbol{C}^+$, such that
$\boldsymbol{C}\boldsymbol{C}^+\boldsymbol{B}=\boldsymbol{B}$. Then $\y$ and $\s$ can not be decoupled from $\f(\cdot,\cdot)$ by the \rm{\textbf{ENC}}$_i(\cdot)$ functions \\
% \\ \hspace*{\fill} \\
$\Rightarrow$ 
$\forall \bold{r} \in S$, $\exists \wP$ {and some constants $\iota_1,\iota_2$},
% \vspace{-0.2cm}
\begin{equation}
\label{eq:zpxright}
\resizebox{\hsize}{!}{$
\bz(\mathbf{r})\wP=\bx(\mathbf{r})+O(\lVert\mathbf{r}-\mathbf{r}_0\rVert^2), \; {\iota_1<\lim_{\mathbf{r}\rightarrow \mathbf{r}_0}\lVert\frac{O(\lVert\mathbf{r}-\mathbf{r}_0\rVert^2)}{\lVert \mathbf{r}-\mathbf{r}_0\rVert^2}\rVert<\iota_2}.
$}
  % \bZ\wP=\bX+o(|\bold{r}-\bold{r}_0|^1)
  \nonumber
\end{equation}
\end{theorem}
% }
% \vspace{-0.3cm}
\begin{proof}
{Since $\bz(\bold{r})$ and $\bx(\bold{r})$ are twice differentiable at $\bold{r}_0$,}
% and their partial derivatives of order $\leq 2$ are differentiable around $\bold{r}_0$, 
we can expand $\bz(\bold{r})$ and $\bx(\bold{r})$ into Taylor series around $\bold{r}_0 \in S$. For $i,j \in \{1,\ldots,\DofHid\}$,
% \yqh{
% \begin{equation}
% \begin{aligned}
%       \x_{\_i}(\bold{r})=&a_{i\_0}+(\bold{r}-\bold{r}_0)^\top\bold{a}_{i\_1}\\
%       &+(\bold{r}-\bold{r}_0)^\top \bold{A}_{i\_2}(\bold{r}-\bold{r}_0)+\cdots
% \end{aligned}
% \end{equation}
% \begin{equation}
% \begin{aligned}
%       \z_{\_j}(\bold{r})=&b_{j\_0}+(\bold{r}-\bold{r}_0)^\top\bold{b}_{j\_1}\\
%       &+(\bold{r}-\bold{r}_0)^\top \bold{B}_{j\_2}(\bold{r}-\bold{r}_0)+\cdots
% \end{aligned}
% \end{equation}}
% \vspace{-0.1cm}
\begin{equation}
     \x_{\_i}(\bold{r})=a_{i\_0}+(\bold{r}-\bold{r}_0)^\top\bold{a}_{i\_1}
      +\oram, \nonumber
      % \vspace{-0.1cm}
\end{equation}
\begin{equation}
      \z_{\_j}(\bold{r})=b_{j\_0}+(\bold{r}-\bold{r}_0)^\top\bold{b}_{j\_1}+\oram, \nonumber
      % \vspace{-0.1cm}
\end{equation}
{where $a_{i\_0}=\x_{\_i}(\bold{r}_{0})$, $b_{j\_0}=\z_{\_j}(\bold{r}_{0})$, and $\bold{a}_{i\_1}$ and $\bold{b}_{j\_1}$ are gradient vectors.}
% \begin{equation}
% \begin{aligned}
%   \z_{\_j}(\y,\s)=&b_{j\_0}+b_{j\_11}(\y-\y_0)+b_{j\_2}(\s-\s_0)\\
%   &+b_{j\_11}(\y-\y_0)^2
%   +b_{j\_12}(\y-\y_0)(\s-\s_0)\\
%   &+b_{j\_22}(\s-\s_0)^2+\cdots
%   \end{aligned}
% \end{equation}

% Referring to \Cref{eq:zpx}, we let:
{We find the matrix $\wP$ by solving the following equations:}
\begin{equation}\label{eq:taylor}
\begin{aligned}
  &\begin{bmatrix}
    a_{1\_0}& \cdots & a_{\DofHid\_0}\\
    \bold{a}_{1\_1}& \cdots & \bold{a}_{\DofHid\_1}
    \end{bmatrix} 
    = \begin{bmatrix}
    b_{1\_0}& \cdots & b_{\DofHid\_0}\\
    \bold{b}_{1\_1}& \cdots & \bold{b}_{\DofHid\_1}
    \end{bmatrix}\wP\\
    &\Leftrightarrow \boldsymbol{B}=\boldsymbol{C}\wP
  \end{aligned}
\end{equation}
% \begin{equation}
%   \left\{
%    \begin{aligned}
%        a_0 &=
%    \begin{bmatrix}
%     b_{1\_0}(\y,\s) & \cdots & b_{\DofHid\_0}(\y,\s)
%    \end{bmatrix}\wP_i \\
%    a_1 &=
%     \begin{bmatrix}
%     b_{1\_1}(\y,\s) & \cdots & b_{\DofHid\_1}(\y,\s)
%    \end{bmatrix}\wP_i \\
%    a_{11} &=
%     \begin{bmatrix}
%     b_{1\_11}(\y,\s) & \cdots & b_{\DofHid\_11}(\y,\s)
%    \end{bmatrix}\wP_i \\
%    &\cdots
%     \end{aligned}
%   \right. ,
% \end{equation}
% \yqh{
Note that there are $\DofHid + 2\DofOut\DofHid$ constraints, whereas $\wP$ has $\DofHid^2$ degrees of freedom. {Using \Cref{thm:ls} and the condition $\boldsymbol{C}\boldsymbol{C}^+\boldsymbol{B}=\boldsymbol{B}$ and $\DofHid\geq 2\DofOut+1$, it follows that the \Cref{eq:taylor} has a feasible solution.}
When the requirements in \Cref{eq:taylor} are met, the first two terms in the Taylor series of $\bz(\bold{r})\wP$ and $\bx(\bold{r})$ are equal. Thus, $\bz(\bold{r})\wP=\bx(\bold{r})+O(\lVert\bold{r}-\bold{r}_0\rVert^2)$.
% the first two elements are the same. Then we can indicate elements that can not be controlled by these equations as $o(|\bold{r}-\bold{r}_0|^1)$. Therefore (\ref{eq:zpx}) holds, and when $\DofHid\rightarrow +\infty$, $o(|\bold{r}-\bold{r}_0|^1)\rightarrow 0$.}
% In our experiment, we set $\DofHid=16$, leading to $n=4$, large enough to guarantee that the infinitesimal amount is small enough.
% \textcolor{purple}{!!!!!!!!!!!This sentence needs the final reasonable setting. Consider put it here or experiment part.}
% \end{proof}
% \end{theorem}
% Following our rationale, if $\bold{P}_{opt}$ is such a matrix satisfying~\Cref{eq:zpx}, $\y$ and $\s$ can not be decoupled from $\f$. Conversely, if $\y$ and $\s$ can not be decoupled from $\f$, there exists a $\bold{P}$ satisfying~\Cref{eq:zpx}. Since $\bold{P}_{opt}$ minimizes $D$ in~\Cref{e4}, $\bold{P}_{opt}$ also satisfies ~\Cref{eq:zpx}.
\end{proof}
% Note that now we only have the property of the opposite loss function. \yqh{To further excavate the relationship between maximizing the loss and decoupling $\f$ theoretically, a technical difficulty is to provide a measure to detect the changing point of the function $\bold{y}(\y,\s)=\bZ(\y,\s)\wP-\bX(\y,\s)$ between an infinitesimal amount $\oram$ and a non-negligible function. Combining the upper bound and Theorem \ref{thm1.1}, we can make sure that decoupling loss indeed works.}

Next we provide a theoretic rationale for maximizing $D$ in \Cref{e4}. We assume that $\y$ and $\s$ are uniformly sampled from their domain $S$. Thus, we have:
\vspace{-0.3cm}
\begin{equation}
\resizebox{\hsize}{!}{$
\label{int}
  {\Bigg\Vert \bold{H}_t   \bold{P}_{opt} - \bold{H}_s \Bigg\Vert_F}=3b\int_{S}\frac{\lVert\bold{y}(\bold{r})\rVert}{V_{S}}d\bold{r}, \; \bold{y}(\bold{r})=\bz(\bold{r})\wP_{opt}-\bx(\bold{r}), $}
  \vspace{0.1cm}
\end{equation}
where $V_{S}=\int_{S}d\bold{r}$, $b$ is the batch size and three is the number of views.
In the following, we prove that if $\bz$, $\bx$ and $\bold{P}_{opt}$ satisfy~\Cref{eq:zpx}, there exists a universal upper bound for \Cref{int} applicable to all $\oram$ {that meet certain criteria.} Since our framework maximizes $D$, the decoupling of $\bz$ and $\bx$ from $\f$ {may be achieved when} the maximized $D$ surpasses the upper bound.

% \yqh{We address this issue by obtaining a universal upper bound of all possible infinitesimal amount. We first introduce the Heine-Borel Covering Theorem and the first mean value theorem for definite integrals \cite{Apostol1974MathematicalA}.}

% \yqh{Now we can prove the existence and uniqueness of the universal upper bound of all possible infinitesimal. We can use the expectation of $\bz$ ($\bx$), where we sample $\y$ and $\s$ ($r$) from a uniform distribution to obtain the actual matrix \bold{H}. Let}
% \begin{equation}\label{eq:zpep}
%          \bold{y}(\y,\s)=\bold{y}(\bold{r})=\bz(\bold{r})\wP-\bx(\bold{r}),
% \end{equation}
% Then
% \begin{equation}\label{int}
%   \Bigg\Vert \bold{H}_s   \bold{P} - \bold{H}_t \Bigg\Vert_F=b\int_{p_1}^{q_1}\cdots\int_{p_{2c}}^{q_{2c}}\frac{y(\bold{r})}{V}dr_1\cdots dr_{2c},
% \end{equation}
% where $V=\prod_{i=1}^{2c}(q_i-p_i)$ and $b$ is the batch size.

% We assume that the domain of $\bold{y}(\bold{r})$ is closed, i.e.$\prod_{i=1}^{2c}{[p_i, q_i]}$. Thus, with the guarantee of \Cref{thm:hb} in the Appendix, there exists a finite open covering $F=\{I_1, \ldots, I_n\}$, where $I_i=\{||\bold{r}-\bold{r}_i||<\delta\}\subseteq \prod_{i=1}^{2c}{[p_i, q_i]}$, $\bold{r}_i\in \prod_{i=1}^{2c}{[p_i, q_i]}$ that can fully cover this domain.
{
\begin{definition}
    Consider a Jordan-measurable set $S'$ and arbitrary constants $\rho_1$ and $\rho_2$, where $0<\rho_1<\rho_2$. 
    We define $\Omega^{S'}_{\rho_1, \rho_2, M_0}$ as the set of functions $\bold{y}(\bold{r})=\oram$ that satisfy the following conditions: $\exists M_0>0$, $3b\int_{S \setminus S'}\frac{\lVert\bold{y}(\bold{r})\rVert}{V_{S}}d\bold{r}<M_0$ and 
$\rho_1\lVert\bold{r}-\bold{r}_0\rVert^2<\lVert\bold{y}(\bold{r})\rVert<\rho_2\lVert\bold{r}-\bold{r}_0\rVert^2$ on $S'$.
\end{definition}
}

\begin{theorem}\label{thm:upb}
{
% Consider a Jordan-measurable set $S'$ and arbitrary constants $\rho_1$ and $\rho_2$, where $0<\rho_1<\rho_2$. 
% We define $\Omega^{S'}_{\rho_1, \rho_2, M_0}$ as the set of functions $\bold{y}(\bold{r})=\oram$ that satisfy the following conditions: $\exists M_0>0$, $3b\int_{S \setminus S'}\frac{\lVert\bold{y}(\bold{r})\rVert}{V_{S}}d\bold{r}<M_0$ and 
% $\rho_1\lVert\bold{r}-\bold{r}_0\rVert^2<\lVert\bold{y}(\bold{r})\rVert<\rho_2\lVert\bold{r}-\bold{r}_0\rVert^2$ on $S'$.
If $\exists$ $\bold{g}_1, \bold{g}_2: \mathbb{R}^{\DofRepr} \mapsto \mathbb{R}^{\DofHid}$ that can decouple the graph representation function $\f$, $\exists$ $U_{\rho_1,\rho_2,M_0}^{S'}\in\mathbb{R}$, which is the universal upper bound for any $\bold{y}(\bold{r}) \in \Omega^{S'}_{\rho_1, \rho_2, M_0}$.}
% \vspace{-2cm}
    % all infinitesimal satisfying Equation (\ref{eq:zpx}).
    \begin{proof}
    {$\forall \varepsilon_0>0$, we define $S_{\varepsilon_0}=\{\bold{r}\mid\lVert\bold{r}-\bold{r}_0\rVert<\frac{\varepsilon_0}{\rho_2}\}$. $\forall \bold{r} \in S_{\varepsilon_0}$, we have:}
    {\begin{equation}
      \begin{aligned}
        \rho_2\lVert\bold{r}-\bold{r}_0\rVert&<{\varepsilon_0}\\
        \lVert\bold{y}(\bold{r})\rVert<\rho_2\lVert\bold{r}-\bold{r}_0\rVert^2&<\varepsilon_0{\lVert\bold{r}-\bold{r}_0\rVert}.
      \end{aligned}
      \end{equation}
    % Since we have assumed \ref{ieq:eta},
    % \begin{equation}
    %     ||\bold{y}(\bold{r})||<\eta_2||\bold{r}-\bold{r}_0||^2<\varepsilon{||\bold{r}-\bold{r}_0||},
    % \end{equation}
    Thus,
    {\begin{equation}
        \frac{\lVert\bold{y}(\bold{r})\rVert}{\lVert\bold{r}-\bold{r}_0\rVert}<\varepsilon_0
    \end{equation}}
    {Since $S'$ is a Jordan-measurable set, $\exists \varepsilon_1$, such that $S_{\varepsilon_1}\supseteq S'.$ Assume that $\varepsilon_0<\varepsilon_1$. Then $\forall \bold{r} \in S'\setminus S_{\varepsilon_0}$, we have:}
    {\begin{equation}
        \frac{\lVert\bold{y}(\bold{r})\rVert}{\lVert\bold{r}-\bold{r}_0\rVert}<\varepsilon_1
    \end{equation}}
    {Referring to \ref{int}, we have:}
    % \begin{equation}\nonumber
      \begin{align*}
        &\Bigg\Vert \bold{H}_t   \bold{P}_{opt} - \bold{H}_s \Bigg\Vert_F\\
        {=}&{3b\int_{S_{\varepsilon_0}}\frac{\lVert\bold{y}(\bold{r})\rVert}{V_{S}}d\bold{r}+3b\int_{S'\setminus S_{\varepsilon_0}}\frac{\lVert\bold{y}(\bold{r})\rVert}{V_{S}}d\bold{r}}\\
        &{+3b\int_{S \setminus S'}\frac{\lVert\bold{y}(\bold{r})\rVert}{V_{S}}d\bold{r}}\\
        <&\frac{3b}{V_{S}}\int_{S_{\varepsilon_0}}\lVert\bold{r}-\bold{r}_0\rVert\frac{\lVert\bold{y}(\bold{r})\rVert}{\lVert\bold{r}-\bold{r}_0\rVert}d\bold{r}\\
        &+\frac{3b}{V_{S}}\int_{S'\setminus S_{\varepsilon_0}}\lVert\bold{y}(\bold{r})\rVert d\bold{r}+M_0\\
        <&\frac{3b}{V_{S}}\int_{S_{\varepsilon_0}}\frac{\varepsilon_0}{\rho_2}\varepsilon_0 d\bold{r}\\
        &+\frac{3b}{V_{S}}\int_{S'\setminus S_{\varepsilon_0}}\varepsilon_1 \lVert\bold{r}-\bold{r}_0\rVert d\bold{r} + M_0\\
        =&\frac{3b\varepsilon_0^2V_{S_{\varepsilon_0}}}{\rho_2V_{S}}+\frac{3b\varepsilon_1}{V_{S}}\lVert\bold{r^\prime}-\bold{r}_0\rVert\int_{S'\setminus S_{\varepsilon_0}}  d\bold{r} + M_0\\
        =&\frac{3b\varepsilon_0^2V_{S_{\varepsilon_0}}}{\rho_2V_{S}}+\frac{3b\varepsilon_1V_{S'\setminus S_{\varepsilon_0}}}{V_{S}}\lVert\bold{r^\prime}-\bold{r}_0\rVert + M_0 \\
        :=& U_{\rho_1,\rho_2,M_0}^{S'}
      \end{align*}
      % \end{equation}
      }where $\bold{r^\prime}$ is some point in $S'\setminus S_{\varepsilon_0}$. In the penultimate step, we apply \Cref{thm:mv} from the Appendix. Since {$V_{S'}=V_{S'\setminus S_{\varepsilon_0}}+V_{S_{\varepsilon_0}}$}, by appropriately selecting $\varepsilon_0$, the above upper bound $U_{\rho_1,\rho_2,M_0}^{S'}$ can be minimized.
    \vspace{-0.2cm}
    \end{proof}
    \vspace{-0.4cm}
\end{theorem}
{In our proposed framework, we aim to maximize $D$. When $D>U_{\rho_1,\rho_2,M_0}^{S'}$, it follows that $\bold{y}(\bold{r}) \notin \Omega^{S'}_{\rho_1, \rho_2, M_0}$. Therefore, as $D$ increases, the likelihood of $\bold{y}(\bold{r})$ satisfying \Cref{eq:zpx} decreases, while the likelihood of decoupling $\bz$ and $\bx$ increases.}

\section{Experiments}

\begin{table*}[h]
\centering
% \vspace{-0.5cm}
\caption{Graph classification performance evaluated on small and large test graphs. Results are reported as average $F_1$ scores along with their standard deviations. The rightmost column shows the average improvements relative to the original performance using the same GNN backbones. For each backbone model and size category (small/large), the best performance is highlighted in red, and the second-best in violet.}
\label{tb:abl}
\resizebox{0.95\linewidth}{!}{
\begin{tabular}{l|cccccccccc}
\toprule
\multirow{2}{*}{\diagbox{\textbf{Models}}{\textbf{Dataset}}} &
% \multicolumn{2}{l}{\textbf{Datasets}} &
  \multicolumn{2}{c}{\textbf{BBBP}} &
  \multicolumn{2}{c}{\textbf{PROTEINS}} &
  \multicolumn{2}{c}{\textbf{GraphSST2}} &
  \multicolumn{2}{c}{\textbf{NCI1}} &
  \multicolumn{2}{c}{\textbf{Avg. Imprv.}} \\ \cline{2-11}
 & \textbf{Small}          & \textbf{Large}          & \textbf{Small}          & \textbf{Large}          & \textbf{Small}          & \textbf{Large}          & \textbf{Small}          & \textbf{Large}          & \textbf{Small}  & \textbf{Large} \\ \midrule
\textbf{GCN}       & 90.79 $_{\pm 0.04}$ & 76.01 $_{\pm 0.03}$ & 71.74 $_{\pm 0.04}$ & 72.71 $_{\pm 0.03}$ & 89.85 $_{\pm 0.02}$ & 83.55 $_{\pm 0.02}$ & 53.11 $_{\pm 0.04}$ & 38.86 $_{\pm 0.07}$ & 0.0 &  0.0\\
\textbf{GCN+SSR}    & \textcolor{violet}{92.17 $_{\pm 0.21}$} & 82.06 $_{\pm 0.04}$ & 72.21 $_{\pm 0.02}$ & 70.87 $_{\pm 0.01}$ & 89.37 $_{\pm 0.04}$ & \textcolor{violet}{84.05 $_{\pm 0.03}$} & 54.37 $_{\pm 0.03}$ & 39.87 $_{\pm 0.02}$ & +0.54 & +2.16 \\
\textbf{GCN+CIGAV2} & 92.00 $_{\pm 0.09}$ & \textcolor{violet}{84.49 $_{\pm 0.07}$} & \textcolor{violet}{72.50 $_{\pm 0.07}$} & \textcolor{violet}{76.67 $_{\pm 0.02}$} & 87.51 $_{\pm 0.02}$ & 83.97 $_{\pm 0.04}$ & \textcolor{red}{56.63 $_{\pm 0.03}$} & \textcolor{violet}{41.22 $_{\pm 0.06}$} & \textcolor{violet}{+1.60} & \textcolor{violet}{+5.79} \\
\textbf{GCN+RPGNN}  & 84.44 $_{\pm 0.04}$ & 75.43 $_{\pm 0.07}$ & 71.91 $_{\pm 0.12}$ & 67.53 $_{\pm 0.18}$ & 90.00 $_{\pm 0.06}$ & 83.30 $_{\pm 0.06}$ & 52.59 $_{\pm 0.02}$ & 40.60 $_{\pm 0.05}$ & -1.89  & -0.93 \\
\textbf{GCN+IRM}   & 91.63 $_{\pm 0.13}$ & 77.33 $_{\pm 0.11}$ & 70.40 $_{\pm 0.31}$ & 74.50 $_{\pm 0.02}$  & \textcolor{red}{90.91 $_{\pm 0.12}$} & 83.28 $_{\pm 0.20}$ & 51.40 $_{\pm 0.08}$ & 39.91 $_{\pm 0.11}$ & -0.75  & +1.68 \\

\textbf{GCN+\method}  & \textcolor{red}{92.92 $_{\pm 0.09}$} & \textcolor{red}{85.63 $_{\pm 0.07}$} & \textcolor{red}{75.00 $_{\pm 0.21}$ }  & \textcolor{red}{83.41 $_{\pm 0.31}$} & \textcolor{violet}{90.08 $_{\pm 0.11}$} & \textcolor{red}{85.18 $_{\pm 0.29}$} &\textcolor{violet}{ 56.14 $_{\pm 0.02}$} & \textcolor{red}{43.46 $_{\pm 0.02}$ }& \textcolor{red}{+2.16} & \textcolor{red}{+6.64} \\ \midrule

\textbf{GIN}         & 86.53 $_{\pm 0.05}$ & 65.17 $_{\pm 0.02}$ & 74.00 $_{\pm 0.07}$   & 74.42 $_{\pm 0.08}$ & 89.76 $_{\pm 0.11}$ & 84.70 $_{\pm 0.22}$ & 47.78 $_{\pm 0.16}$ & 30.94 $_{\pm 0.10}$ & 0.00 & 0.00 \\
\textbf{GIN+SSR}     & 81.52 $_{\pm 0.08}$ & 70.44 $_{\pm 0.01}$ & 70.20 $_{\pm 0.05}$ & 75.06 $_{\pm 0.04}$ & \textcolor{violet}{90.41 $_{\pm 0.07}$} & 83.43 $_{\pm 0.44}$ & 50.11 $_{\pm 0.32}$ & 32.61 $_{\pm 0.12}$ & -1.33 & \textcolor{violet}{+3.21} \\
\textbf{GIN+CIGAV2}  & \textcolor{red}{88.46 $_{\pm 0.07}$ } & \textcolor{violet}{72.12 $_{\pm 0.02}$} & \textcolor{red}{75.44 $_{\pm 0.07}$} & \textcolor{violet}{76.67 $_{\pm 0.10}$} & 89.51 $_{\pm 0.06}$ & 84.25 $_{\pm 0.06}$ & 46.56 $_{\pm 0.04}$ & \textcolor{violet}{32.70 $_{\pm 0.33}$} & +0.33 & +2.63 \\
\textbf{GIN+RPGNN}   & 82.19 $_{\pm 0.22}$ & 70.64 $_{\pm 0.41}$ & 74.08 $_{\pm 0.05}$ & 71.03 $_{\pm 0.08}$ & 89.00 $_{\pm 0.04}$ & 83.83 $_{\pm 0.04}$ & \textcolor{red}{51.40 $_{\pm 0.03}$ }& 31.87 $_{\pm 0.32}$ & \textcolor{violet}{+0.46} & +1.45 \\
\textbf{GIN+IRM}    & 86.09 $_{\pm 0.31}$ & 65.27 $_{\pm 0.50}$ & 75.27 $_{\pm 0.02}$ & 75.07 $_{\pm 0.09}$ & \textcolor{red}{90.43 $_{\pm 0.04}$ }& \textcolor{violet}{84.84 $_{\pm 0.08}$} & 47.52 $_{\pm 0.06}$ & 30.02 $_{\pm 0.08}$ & +0.35 & -0.44 \\
\textbf{GIN+\method}    & \textcolor{violet}{87.15} $_{\pm 0.08}$ & \textcolor{red}{72.45$_{\pm 0.07}$} & \textcolor{violet}{75.30 $_{\pm 0.02}$ }& \textcolor{red}{78.25 $_{\pm 0.06}$} & 90.30 $_{\pm 0.12}$ & \textcolor{red}{84.87 $_{\pm 0.32}$} & \textcolor{violet}{50.28 $_{\pm 0.42}$} & \textcolor{red}{33.02 $_{\pm 0.33}$} & \textcolor{red}{+1.24} & \textcolor{red}{+3.34} \\ \midrule
\textbf{GT}          & 89.78 $_{\pm 0.02}$ & 83.37 $_{\pm 0.08}$ & 70.22 $_{\pm 0.01}$ & 71.62 $_{\pm 0.02}$ & \textcolor{red}{90.80} $_{\pm 0.04}$ & 83.33 $_{\pm 0.04}$ & 59.00 $_{\pm 0.12}$ & 41.50 $_{\pm 0.19}$ & 0.00 & 0.00 \\
\textbf{GT+SSR}      & \textcolor{violet}{91.85 $_{\pm 0.01}$} & 77.65 $_{\pm 0.02}$ & \textcolor{red}{73.91 $_{\pm 0.02}$} & \textcolor{violet}{73.73 $_{\pm 0.06}$} & 90.46 $_{\pm 0.12}$ & 83.44 $_{\pm 0.02}$ & 59.17 $_{\pm 0.01}$ & \textcolor{violet}{42.80 $_{\pm 0.01}$} & \textcolor{violet}{+1.87}  & -0.16 \\
\textbf{GT+CIGAV2}   & 91.43 $_{\pm 0.31}$ & 84.75 $_{\pm 0.20}$ & 70.21 $_{\pm 0.03}$ & 73.01 $_{\pm 0.04}$ & 89.47 $_{\pm 0.05}$ & \textcolor{violet}{84.56 $_{\pm 0.03}$ }& \textcolor{violet}{59.86} $_{\pm 0.03}$ & 39.23 $_{\pm 0.03}$ & +0.45  & -0.10 \\
\textbf{GT+RPGNN}    & 90.98 $_{\pm 0.19}$ & 83.16 $_{\pm 0.31}$ & 71.73 $_{\pm 0.07}$ & 70.00 $_{\pm 0.10}$ & 90.72 $_{\pm 0.43}$& 84.32 $_{\pm 0.02}$ & 58.82 $_{\pm 0.03}$  & 42.50 $_{\pm 0.04}$ & +0.77 & \textcolor{violet}{+0.27} \\
\textbf{GT+IRM}      & 91.50 $_{\pm 0.02}$ & \textcolor{violet}{85.51 $_{\pm 0.32}$} & 72.53 $_{\pm 0.01}$ & 71.68 $_{\pm 0.11}$ & 90.40 $_{\pm 0.56}$ & 83.20 $_{\pm 0.04}$ & 58.08 $_{\pm 0.04}$ & 35.69 $_{\pm 0.05}$ & +0.80 & -2.87 \\
\textbf{GT+\method}     & \textcolor{red}{93.77 $_{\pm 0.01}$ } & \textcolor{red}{88.01 $_{\pm 0.02}$} & \textcolor{violet}{73.06$_{\pm 0.03}$}   & \textcolor{red}{79.22 $_{\pm 0.03}$} & \textcolor{violet}{90.77 $_{\pm 0.02}$} & \textcolor{red}{85.03 $_{\pm 0.15}$} & \textcolor{red}{62.19 $_{\pm 0.03}$} & \textcolor{red}{45.54 $_{\pm 0.02}$} & \textcolor{red}{+2.50} & \textcolor{red}{+4.50} \\ \bottomrule
% \label{tb:performance}
\end{tabular}
}
\vspace{-0.2cm}
\end{table*}

\begin{table*}[h]
\centering
\caption{Impact of different design choices on graph classification performance across small and large test graphs. Performance is assessed by average $F_1$ scores and their standard deviations. The rightmost column shows the average improvements relative to \method. For each size category (small/large), the model with the highest performance is highlighted in red.}
\resizebox{0.95\linewidth}{!}{
\begin{tabular}{l|cccccccccc}
\toprule
\multirow{2}{*}{\diagbox[width=15em]{\textbf{Models}}{\textbf{Dataset}}} &
% \multicolumn{2}{l}{\textbf{Datasets}} &
  \multicolumn{2}{c}{\textbf{BBBP}} &
  \multicolumn{2}{c}{\textbf{PROTEINS}} &
  \multicolumn{2}{c}{\textbf{GraphSST2}} &
  \multicolumn{2}{c}{\textbf{NCI1}} &
  \multicolumn{2}{c}{\textbf{Avg. Imprv.}} \\ \cline{2-11}

& \textbf{Small}          & \textbf{Large}          & \textbf{Small}          & \textbf{Large}          & \textbf{Small}          & \textbf{Large}          & \textbf{Small}          & \textbf{Large}          & \textbf{Small}  & \textbf{Large} \\
\midrule
\textbf{\method}   &  \textcolor{red}{92.92 $_{\pm 0.09}$} & \textcolor{red}{85.63 $_{\pm 0.07}$} & \textcolor{red}{75.00 $_{\pm 0.21}$}   & \textcolor{red}{83.41 $_{\pm 0.31}$} & \textcolor{red}{90.08 $_{\pm 0.11}$} & \textcolor{red}{85.18 $_{\pm 0.29}$} & \textcolor{red}{56.14 $_{\pm 0.02}$} & \textcolor{red}{43.46 $_{\pm 0.02}$} & \textcolor{red}{0} & \textcolor{red}{0} \\

\textbf{w/o Aug.}  &    88.51 $_{\pm 0.01}$ & 80.29 $_{\pm 0.04}$ & 73.34 $_{\pm 0.05}$ & 72.33 $_{\pm 0.01}$ & 84.14 $_{\pm 0.02}$ & 82.62 $_{\pm 0.01}$ & 54.22 $_{\pm 0.04}$ & 38.99 $_{\pm 0.13}$ & -3.48 & -5.86 \\

\textbf{w/o Decpl.}  & 86.33 $_{\pm 0.04}$ & 78.72 $_{\pm 0.02}$ & 77.43 $_{\pm 0.02}$ & 80.41 $_{\pm 0.01}$ & 88.33 $_{\pm 0.02}$ & 83.95 $_{\pm 0.04}$ & 55.22 $_{\pm 0.03}$ & 32.98 $_{\pm 0.08}$ & -1.71 & -5.41 \\

% \textbf{w/o Decpl..}    &
\textbf{w/o Decpl.+Decorr.}  & 

90.34 $_{\pm 0.02}$ & 73.28 $_{\pm 0.01}$ & 77.11 $_{\pm 0.03}$ & 81.82 $_{\pm 0.03}$ & 88.81 $_{\pm 0.02}$ & 84.21 $_{\pm 0.04}$ & 54.25 $_{\pm 0.02}$ & 36.14 $_{\pm 0.11}$ & -0.91  & -5.56 \\

\textbf{w/o (Decpl.+ Aug.)+Decorr.}   & 
% \textbf{w/o (Decpl.}     &
88.76 $_{\pm 0.00}$ & 72.54 $_{\pm 0.00}$ & 76.13 $_{\pm 0.00}$ & 82.35 $_{\pm 0.01}$ & 88.53 $_{\pm 0.02}$ & 85.12 $_{\pm 0.18}$ & 55.35 $_{\pm 0.01}$ & 35.61 $_{\pm 0.02}$ & -1.34  & -5.51 \\
\bottomrule
\end{tabular}
}
\vspace{-0.2cm}
\end{table*}

In this section, we conduct extensive experiments to evaluate our proposed framework \method. We aim to answer the following questions: (\textbf{RQ1}) Does \method effectively enhance the size generalizability of GNNs, and how does it compare to other baselines? (\textbf{RQ2}) how do various components influence the size generalizability of \method? 
% In this section, we conduct experiments on four public datasets to evaluate the proposed method and ablation studies to investigate the influence of each component.

\subsection{Experimental Setup}

\textbf{Datasplits.} Each dataset is divided into four subsets: training, validation, small test, and large test. The small test sets include graphs similar in size to the training set, while the large test sets contain significantly larger graphs. The splits are generated by first sorting the dataset samples by size. The training, validation, and small test subsets are randomly selected from the smallest 50\% of graphs, while the large test subsets are chosen from the largest 10\% of graphs. Further details can be found in Appendix \ref{app:data pre-processing}.

\textbf{Datasets.} We perform experiments on four datasets: BBBP \cite{wu2018moleculenet}, PROTEINS \cite{hobohm1992selection}, GraphSST2 \cite{yuan2021explainability}, and NCI1 \cite{xinyi2018capsule}. Each method is evaluated using the $F_1$ score. Details about the datasets
% % In order to analyze size generalizability, we have four splits for each dataset: 
% train, validation, small\_test, and large\_test, where large\_test contains graphs with significantly larger sizes. We generate the splits as follows. First, we sort the samples in the dataset by their size. Next, We take the train, validation, and small\_test split from the 50\% smallest graphs in the dataset. 
% \yy{More details about datasplitting can be found in Appendix [ref].}
% We conduct experiments on four datasets, BBBP \cite{wu2018moleculenet}, PROTEINS \cite{hobohm1992selection}, GraphSST2 \cite{yuan2021explainability} and NCI1 \cite{xinyi2018capsule}. We evaluate the performance of each method on $F_1$ score. The preprocessing and statistics of the dataset 
can be found in Appendix \ref{app:dataset}. {We also test our method on two larger datasets and report the prediction performance in Appendix \ref{Two more real-world datasets}.  }

\textbf{Baselines.} We assess the performance of \method against four baseline approaches for graph size generalization: SizeShiftReg (SSR) \cite{buffelli2022sizeshiftreg}, CIGAV2 \cite{chen2022learning}, RPGNN \cite{murphy2019relational}, and IRM \cite{arjovsky2019invariant}. We utilize different GNN backbones—GCN \cite{kipf2016semi}, GIN \cite{xu2018powerful}, and GraphTransformer \cite{shi2020masked}—in both the mentioned approaches and our framework. { The training details and hyperparameter choices can be found in Appendix \ref{app:Training Details}.}
% \textbf{COMPARISON METHODS. } We compare the performance of the proposed framework with graph size generalization approaches: SizeShiftReg \cite{buffelli2022sizeshiftreg}, CIGAV2 \cite{chen2022learning}, RPGNN\cite{murphy2019relational} and IRM \cite{arjovsky2019invariant}. We change GNN backbone of the above methods and our framework from GCN \cite{kipf2016semi}, GIN \cite{xu2018powerful} and graph transformer \cite{shi2020masked}. 

\subsection{Effectiveness of \method}
We assess the graph prediction performance of \method on both small and large test graphs, comparing it against four baseline methods. The results are summarized in Table \ref{tb:abl}. Our results reveal two key insights:

First, across various datasets, \method consistently improves the graph prediction performance of various GNN backbones on both small and large test graphs. Notably, \method achieves an average improvement of up to 2.50\% on small test graphs and up to 6.64\% on large test graphs. These results highlight the effectiveness of \method in enhancing both in-distribution and size generalization.
% First, our disentanglement framework significantly boosts the performance of various GNN backbones. 
% This improvement is particularly evident in models such as 'GCN+Dis.', 'GIN+Dis.', and 'GT+Dis.', which consistently outperformed their respective base models (GCN, GIN, GT) across all tested scenarios. A notable observation is the average performance improvement of 8.44\% with the GCN backbone. These results underscore the effectiveness of our framework in leveraging relative size information among different views and successfully decoupling size-related and task-related information from the learned graph representation.
% Second, compared to other baselines, \yy{\method} exhibits superior performance, particularly in scenarios involving larger graphs. 

Second, \method surpasses other baselines by achieving the highest improvements in $F_1$ scores for both small and large test graphs, with a more pronounced impact on the latter. Compared to other competitive baselines, such as SSR \cite{buffelli2022sizeshiftreg} and CIGAV2 \cite{chen2022learning}, \method's superior efficacy stems from its disentangled learning approach. Through explicit removal of size-related information from task-related representations, \method consistently exhibits improvements across various GNN backbones. In contrast, SSR and CIGAV2 fail to improve the size generalizability for the graph transformer model \cite{shi2020masked}.
\begin{table*}[h]
\centering
\caption{Impact of different augmentations on graph classification performance across small and large test graphs. Performance is assessed by $F_1$ scores and their standard deviations. The rightmost column shows the average improvements relative to \method. For each
size category (small/large), the model with the highest performance is highlighted in red.}
\resizebox{0.95\textwidth}{!}{
\begin{tabular}{l|cccccccccc}
\toprule
\multirow{2}{*}{\diagbox[width=10em]{\textbf{Models}}{\textbf{Dataset}}} &

% \textbf{Dataset} &
  \multicolumn{2}{c}{\textbf{BBBP}} &
  \multicolumn{2}{c}{\textbf{PROTEINS}} &
  \multicolumn{2}{c}{\textbf{GraphSST2}} &
  \multicolumn{2}{c}{\textbf{NCI1}} &
  \multicolumn{2}{c}{\textbf{Avg. Imprv.}} \\ \cline{2-11}
 &
  \textbf{Small} &
  \textbf{Large} &
  \textbf{Small} &
  \textbf{Large} &
  \textbf{Small} &
  \textbf{Large} &
  \textbf{Small} &
  \textbf{Large} &
  \textbf{Small} &
  \textbf{Large} \\ 
\midrule

\textbf{\method} &
    \textcolor{red}{92.92 $_{\pm 0.09}$} &
    \textcolor{red}{85.63 $_{\pm 0.07}$} &
    \textcolor{red}{75.00 $_{\pm 0.21}$} &
    \textcolor{red}{83.41 $_{\pm 0.31}$} &
    \textcolor{red}{90.08 $_{\pm 0.11}$} &
    \textcolor{red}{85.18 $_{\pm 0.29}$} &
    \textcolor{red}{56.14 $_{\pm 0.02}$} &
    \textcolor{red}{43.46 $_{\pm 0.02}$} &
    \textcolor{red}{0} &
    \textcolor{red}{0} \\
\textbf{W/o\_size-inv\_branch} &
  91.00 $_{\pm 0.02}$ &
  82.45 $_{\pm 0.06}$ &
  73.35 $_{\pm 0.03}$ &
  81.11 $_{\pm 0.03}$ &
  89.46 $_{\pm 0.01}$ &
  83.38 $_{\pm 0.02}$ &
  56.28 $_{\pm 0.03}$ &
  39.03 $_{\pm 0.02}$ &
  -1.01 &
  -2.93 \\
\textbf{W/o\_task-inv\_branch} &
  90.12 $_{\pm 0.02}$ &
  81.01 $_{\pm 0.01}$ &
  72.39 $_{\pm 0.03}$ &
  80.76 $_{\pm 0.06}$ &
  89.48 $_{\pm 0.01}$ &
  82.33 $_{\pm 0.03}$ &
  54.10 $_{\pm 0.04}$ &
  36.99 $_{\pm 0.10}$ &
  -2.01 &
  -4.15 \\
  \bottomrule
\end{tabular}
}
\label{tb:abl_aug}
\vspace{-5mm}
\end{table*}

\subsection{Ablation Study}
\label{abl}
{\textbf{Design Choices. } }In this subsection, we study how various design choices impact the size generalizability of \method. For simplicity, we employ GCN~\cite{kipf2016semi} as the backbone model, though our findings are applicable to other GNN backbones. The following design choices are considered: \textbf{w/o Aug.}, we remove the augmentation components from \method; \textbf{w/o Decpl.}, we exclude the decoupling loss in \method; \textbf{w/o Decpl.+Decorr.}, we replace the decoupling loss with a widely used decorrelation loss that enforces cosine similarity to approach 0; and \textbf{w/o (Decpl.+Aug.)+Decorr.}, we apply the same decorrelation loss to \method and remove both augmentations and decoupling loss. The results for different design choices are presented in Table \ref{tb:abl}. In general, the removal or replacement of the augmentations and decoupling loss results in significant performance degradation on both small and large test graphs, with a more pronounced effect on the latter. This emphasizes the effectiveness of these two designs. Specifically, excluding the decoupling loss leads to a reduction of 5.41\% in $F_1$ on large test graphs, and this reduction cannot be compensated for by adding a common decorrelation loss, as demonstrated by the performance of \textbf{w/o Decpl.+Decorr.}. This underscores the crucial role of the decoupling loss in effectively disentangling size- and task-related information. Additionally, augmentations also contribute to enhancing size generalizability. This is evident in the performance of \textbf{w/o Aug.}, where the removal of augmentations results in a decline of 5.86\% in $F_1$ on large test graphs. A similar trend is observed in the performance of \textbf{w/o (Decpl.+Aug.)+Decorr.}, where the performance on large test graphs experiences a further decline compared to \textbf{w/o Decpl.+Decorr.}.

{\textbf{Augmentations. } } We evaluate the impact of size- and task-invariant augmentations on the size generalizability of \method. Specifically, we create two models: one without the task-invariant branch (denoted by \textbf{w/o\_task-inv\_branch}) and another without the size-invariant branch (denoted by \textbf{w/o\_size-inv\_branch}). Due to the elimination of one branch, employing the contrastive loss becomes impractical since it needs both size- and task-invariant views for comparison. Consequently, we modify the contrastive loss. In \textbf{w/o\_task-inv\_branch} model, we replace the contrastive loss with a loss to maximize the cosine similarity of the size representations $s_i^{(1)}$ and $s_i$; and in \textbf{w/o\_size-inv\_branch} model, we use a loss to minimize the cosine similarity of $s_i^{(2)}$ and $s_i$. We present their results in Table \ref{tb:abl_aug}. As can be seen from the results, removing any augmentation leads to a performance decline on both small and large test graphs.

\begin{table}[]
\caption{Impact of pre-trained models fed to PGExplainer, assessed by average $F_1$ scores and their standard deviations.}
\resizebox{\linewidth}{!}{
\begin{tabular}{l|cc|cc}
\toprule
\multirow{2}{*}{\diagbox{\textbf{Dataset}}{\textbf{Models}}} &
% \textbf{\begin{tabular}[c]{@{}l@{}}Pretrained model \\ fed to PGExplainer\end{tabular}} & 
\multicolumn{2}{c|}{\textbf{GIN}} & 
\multicolumn{2}{c}{\textbf{GT}} \\ \cline{2-5}
%\textbf{Dataset}   &
&\textbf{Small}               & \textbf{Large}               & \textbf{Small}               & \textbf{Large}               \\ \midrule
\textbf{BBBP}      & 92.22 $_{\pm 0.09}$ & 85.69 $_{\pm 0.03}$ & 91.88 $_{\pm 0.08}$ & 85.01 $_{\pm 0.03}$ \\
\textbf{PROTEINS}  & 75.01 $_{\pm 0.12}$ & 83.49 $_{\pm 0.09}$ & 74.99 $_{\pm 0.51}$ & 84.99 $_{\pm 0.02}$ \\
\textbf{GraphSST2} & 89.99 $_{\pm 0.08}$ & 85.03 $_{\pm 0.31}$ & 89.33 $_{\pm 0.05}$ & 85.21 $_{\pm 0.22}$ \\
\textbf{NCI1}      & 56.90 $_{\pm 0.05}$ & 43.51 $_{\pm 0.11}$ & 57.16 $_{\pm 0.21}$ & 45.71 $_{\pm 0.05}$ \\ \bottomrule
\end{tabular}
}
\label{Pretrained model fed main}
\vspace{-0.5cm}
\end{table}

{\textbf{Pre-trained GNNs. } } To minimize the effect of knowledge distillation \cite{yoon2022semi} from augmented views, such as task-invariant views, we employ a simple model, GCN \cite{kipf2016semi}, in the explainable graph model. To evaluate the influence of other pre-trained models, we conduct experiments using PGExplainer \cite{luo2020parameterized} with different pre-trained models (GIN, and GraphTransformer) to generate task-invariant views. Table \ref{Pretrained model fed main} presents the performance of our methods with various pre-trained models. The results indicate that (1) our model consistently performs well with different pre-trained models for augmentation, and (2) using superior pre-trained models, such as GT, may further enhance our method's performance compared to the results in Table \ref{tb:abl}.

Therefore, the observed enhancements in our model are not due to knowledge distillation from the pre-trained models, as we used the least effective pre-trained model. Instead, these improvements result from our model's ability to separate size and task information.

\section{Related Work}
% This section briefly reviews the topics related to this work, including graph size generalization and disentangled representation learning.

\textbf{Size Generalization on Graphs. }
%Many works have noticed poor size-generalization in GNNs \cite{joshi2022learning, gasteiger2022gemnet, yan2023size, velivckovic2017graph}. 
To enhance the size generalizability of GNNs, \citet{yehudai2021local} introduce the concept of d-patterns and propose a self-supervised framework to address discrepancies in d-patterns between small and large graphs. \citet{bevilacqua2021size} design a size-invariant causal method by modeling the generative process for graphs in the dataset. \citet{chen2022learning} propose a model that captures the invariance of graphs under various distribution shifts. 
%However, the above-mentioned methodologies either need access to test distribution or assume the causal model is correct, which is not guaranteed in real-world scenarios. 
% \citet{knyazev2019understanding} found that using attention with proper thresholding can improve the size generalizability. 
\citet{buffelli2022sizeshiftreg} propose to simulate the graph size shifts using graph coarsening methods. \citet{chu2023wasserstein} tackle the size generalization problem with a Wasserstein barycenter matching layer, which represents an input graph using Wasserstein distances between its node embeddings and learned class-wise barycenters \cite{keriven2020convergence, agueh2011barycenters}.
% , and shows the convergence rate. 
%However, compared to our approach that decouples size- and task-related information to boost size generalizability, the above-mentioned methods either need access to test distribution or assume the causal model is correct, which is not guaranteed in real-world scenarios. 

\textbf{Disentangled Representation Learning. }
% Disentangled representation learning aims to learn representations that identify and disentangle the factors hidden in the given data. 
Existing disentangled representation learning has made significant progress in identifying and separating hidden factors in various fields. \cite{higgins2016beta, xu2021multi, sarhan2020fairness, creager2019flexibly}.
% For example, \cite{sarhan2020fairness} disentangle the meaningful and sensitive representations by enforcing orthogonality constraints in the VAE framework. DDPAE \cite{jiang2020psgan} proposes a disentangled predictive auto-encoder framework to learn both the latent decomposition and disentanglement without explicit supervision. S3VAE \cite{zhu2020s3vae} proposes a self-supervised sequential VAE model that disentangles the time-varying variables and time-invariant variables of video and audio sequences. 
In the graph domain, GraphLoG \cite{xu2021self} employs a self-supervised learning framework to disentangle local similarities and global semantics. DGCL \cite{li2021disentangled} disentangles graph-level representations by ensuring that the factorized representations independently capture expressive information from different latent factors. Additionally, \citet{mo2023disentangled} address the complex relationships between nodes, designing a framework to separate task-relevant from task-irrelevant information in multiplex graphs.
%Although achieving excellent results, the above-mentioned methods are not applicable to distinguish size- and task-related information in graph representations. 

\textbf{Graph Augmentations. }
To address the problem of data insufficiency and improve the data quality,
% improve the sufficiency and quality of training graphs, 
graph
augmentation methods are proposed to generate new graphs
% as an effective tool to augment
% the given graph 
by either slightly modifying the existing
data samples or generating synthetic ones \cite{trivedi2022augmentations, ding2022data, shorten2019survey}.
% Specifically, graph augmentation aims to find a transformation function $f_{\theta}(\cdot):G \to \tilde{G} $ to generate augmented graphs $\{\Tilde{G_i} = (\bold{A}^{(\Tilde{G_i})}, \bold{X}^{\Tilde{(G_i)}}) \}$ that can enrich or enhance the preserved information from the given graph \cite{mo2022simple, you2021graph}. 
Perturbations on graph structures, e.g., adding or dropping edges, are widely adopted augmentation methods \cite{velivckovic2018deep, you2020graph, zhu2021graph}. Another line of research employs explainable models to identify the key structures of the input graphs, guiding the augmentation process for acquiring more effective unsupervised representations \cite{shi2023engage, wang2021molecular}. 
%Our method utilizes graph augmentation to boost the learning of relative size information 

\section{Conclusion}

In this paper, we propose a novel framework \method to enhance the size generalization of GNNs with disentangled representation learning. We first utilize size- and task-invariant augmentations to guide the model in learning relative size information. Additionally, we design a decoupling loss to minimize
the shared information with theoretical guarantees, which effectively disentangle task- and size-related information. Comprehensive experimental results demonstrate that our \method consistently outperforms state-of-the-art methods.

\section*{Impact Statement}

This paper presents work whose goal is to advance the field of machine learning. Our work improves the generalizability of GNNs across different sizes, potentially benefiting sectors such as healthcare, bioinformatics, and program synthesis. We anticipate no direct negative societal or ethical implications from our research.

% In the unusual situation where you want a paper to appear in the
% references without citing it in the main text, use \nocite
% \nocite{langley00}

% \bibliography{ref}

\bibliographystyle{icml2024}

%%%%%%%%%%%%%%%%%%%%%%%%%%%%%%%%%%%%%%%%%%%%%%%%%%%%%%%%%%%%%%%%%%%%%%%%%%%%%%%
%%%%%%%%%%%%%%%%%%%%%%%%%%%%%%%%%%%%%%%%%%%%%%%%%%%%%%%%%%%%%%%%%%%%%%%%%%%%%%%
% APPENDIX
%%%%%%%%%%%%%%%%%%%%%%%%%%%%%%%%%%%%%%%%%%%%%%%%%%%%%%%%%%%%%%%%%%%%%%%%%%%%%%%
%%%%%%%%%%%%%%%%%%%%%%%%%%%%%%%%%%%%%%%%%%%%%%%%%%%%%%%%%%%%%%%%%%%%%%%%%%%%%%%
\newpage
\appendix
\onecolumn

\section{Lemmas in Theoretical Analysis}
\begin{lemma}(Solutions of Linear Systems~\cite{BenIsrael1974GeneralizedIT})\label{thm:ls}
    Let $\mathbf{C} \in \mathbb{C}^{m \times n}, \mathbf{G} \in \mathbb{C}^{p \times q}, \mathbf{B} \in \mathbb{C}^{m \times q}$. Then the matrix equation
$$
\mathbf{C} \mathbf{P} \mathbf{G}=\mathbf{B}
$$
is consistent if and only if, for some $\mathbf{C}^{+}, \mathbf{G}^{+}$,
$$
\mathbf{C} \mathbf{C}^{+} \mathbf{B} \mathbf{G}^{+} \mathbf{G}=\mathbf{B},
$$
in which case the general solution is
$$
\mathbf{P}=\mathbf{C}^{+} \mathbf{B} \mathbf{G}^{+}+\mathbf{Y}-\mathbf{C}^{+} \mathbf{C} \mathbf{Y} \mathbf{G} \mathbf{G}^{+}
$$
for arbitrary $\mathbf{Y} \in \mathbb{C}^{n \times p}$.
\end{lemma}
{In our case, $\mathbf{G}$ and $\mathbf{G}^{+}$ are identity matrices.}
% Referring to \cite{Apostol1974MathematicalA},
% \begin{lemma}(Heine-Borel)\label{thm:hb}
%     Let $F$ be an open covering of a closed and bounded set $A$ in $\mathbb{R}^{2\DofOut}$. Then a finite sub-collection of $F$ also covers $A$.
% \end{lemma}
% \begin{lemma}(1st mean value theorem)\label{thm:mv}
%     if $f:\prod_{i=1}^{2c}{[p_i, q_i]} \mapsto \mathbb{R}$ is continuous and $g$ is an integrable function that does not change sign on $\prod_{i=1}^{2c}{[p_i, q_i]}$, then there exists $\bold{\xi}$ in $\prod_{i=1}^{2c}{(p_i, q_i)}$ such that
% \begin{equation}
% \int_{p_1}^{q_1}\cdots\int_{p_{2c}}^{q_{2c}} f(\bold{x}) g(\bold{x}) d x=f(\bold{\xi}) \int_{p_1}^{q_1}\cdots\int_{p_{2c}}^{q_{2c}} g(\bold{x}) dx_1\cdots dx_{2c},
% \end{equation}
% where $\bold{\xi}$ is a point in $\prod_{i=1}^{2c}{[p_i, q_i]}$.
% \end{lemma}
% \begin{lemma}(1st mean value theorem)\label{thm:mv}
%     if $I$ is a closed set in $\mathbb{R}^{2c}$, $f:\mathbb{R}^{2c} \mapsto \mathbb{R}$ is continuous and $g$ is an integrable function that does not change sign on $I$, then there exists $\bold{\xi} \in I$ such that
% \begin{equation}
% \int_{I} f(\bold{x}) g(\bold{x}) d \bold{x}=f(\bold{\xi}) \int_{I} g(\bold{x}) d \bold{x},
% \end{equation}
% where $\bold{\xi}$ is a point in $I$.
% \end{lemma}
\begin{lemma}(Mean-Value Theorem for multiple integrals, Theorem 14.16 in \citet{Apostol1974MathematicalA})\label{thm:mv}
Assume that $g \in \mathbb{R}$ and $f \in \mathbb{R}$ on a Jordan-measurable set $S$ in $\mathbb{R}^n$ and suppose that $g(\mathbf{x}) \geq 0$ for each $\mathbf{x}$ in $S$. Let $m=\inf f(S), M=\sup f(S)$. Then there exists a real number $\lambda$ in the interval $m \leq \lambda \leq M$ such that
\begin{equation}\label{eq:lem2.1}
    \int_S f(\mathbf{x}) g(\mathbf{x}) d \mathbf{x}=\lambda \int_S g(\mathbf{x}) d \mathbf{x} .
\end{equation}
In particular, we have
\begin{equation}
    m c(S) \leq \int_S f(\mathbf{x}) d \mathbf{x} \leq M c(S),
\end{equation} where $c(S)$ represents the area of set $S$ in $\mathbb{R}^n$.
\\
If, in addition, $S$ is connected and $f$ is continuous on $S$, then $\lambda=f\left(\mathbf{x}_0\right)$ for some $\mathbf{x}_0$ in $S$ and \Cref{eq:lem2.1} becomes
\begin{equation}\label{eq:lem2.2}
    \int_S f(\mathbf{x}) g(\mathbf{x}) d \mathbf{x}=f\left(\mathbf{x}_0\right) \int_S g(\mathbf{x}) d \mathbf{x}.
\end{equation}
In particular, \Cref{eq:lem2.2} implies $\int_S f(\mathbf{x}) d \mathbf{x}=f\left(\mathbf{x}_0\right) c(S)$, where $\mathbf{x}_0 \in S$.
\end{lemma}
% You can have as much text here as you want. The main body must be at most $8$ pages long.
% For the final version, one more page can be added.
% If you want, you can use an appendix like this one.  

% The $\mathtt{\backslash onecolumn}$ command above can be kept in place if you prefer a one-column appendix, or can be removed if you prefer a two-column appendix.  Apart from this possible change, the style (font size, spacing, margins, page numbering, etc.) should be kept the same as the main body.
%%%%%%%%%%%%%%%%%%%%%%%%%%%%%%%%%%%%%%%%%%%%%%%%%%%%%%%%%%%%%%%%%%%%%%%%%%%%%%%
%%%%%%%%%%%%%%%%%%%%%%%%%%%%%%%%%%%%%%%%%%%%%%%%%%%%%%%%%%%%%%%%%%%%%%%%%%%%%%%
\section{Dataset}
\label{app:dataset}

In this section, we introduce the datasets used in our study. We utilize four pre-processed datasets for the graph classification task: BBBP from the Open Graph Benchmark \cite{hu2020open}, PROTEINS and NCI1 from the TuDataset \cite{morris2020tudataset}, and the GraphSST2 \cite{yuan2022explainability} dataset. Below is a detailed description of each dataset:
\begin{itemize}
\item \textbf{BBBP}: The Blood-Brain Barrier Penetration (BBBP) dataset originates from a study on modeling and predicting barrier permeability. It represents molecules as graphs where nodes are atoms and edges are chemical bonds. Each node features a 9-dimensional vector including atomic number, chirality, formal charge, and ring presence, among other attributes \cite{hu2020open, wu2018moleculenet}. The dataset comprises over 2000 compounds with binary labels indicating their permeability properties.\\

\item \textbf{PROTEINS}: This dataset contains macromolecular graphs of proteins, with nodes corresponding to amino acids. Edges connect nodes that are less than six Angstroms apart. Node features include three-dimensional vectors representing the type of secondary structure elements (helix, sheet, or turn). The dataset provides binary labels for protein functionality (enzyme or non-enzyme) and includes 1113 samples.\\

\item \textbf{GraphSST2}: A real-world dataset for natural language sentiment analysis. Sentences are transformed into grammar tree graphs using the Biaffine parser \cite{gardner2018allennlp}, with each node representing a word associated with corresponding 768-dimensional word embeddings. The dataset's binary classification task involves predicting the sentiment polarity of sentences and includes a total of 35909 samples.\\

\item \textbf{NCI1}: This dataset includes chemical compounds screened for their activity against non-small cell lung cancer. Graphs represent chemical compounds, where nodes denote atoms with one-hot encoded features for atom types, and edges represent chemical bonds. The dataset contains 4110 samples.
\end{itemize}

\begin{table}[ht]
\centering
\begin{tabular}{llcccc}
\toprule
\textbf{Dataset} &
   &
  \multicolumn{1}{l}{\textbf{BBBP}} &
  \multicolumn{1}{l}{\textbf{PROTEINS}} &
  \multicolumn{1}{l}{\textbf{GraphSST2}} &
  \multicolumn{1}{l}{\textbf{NCI1}} \\ \midrule
\textbf{Train}      &  & 1026 & 389 & 21857 & 1438 \\
\textbf{Validation} &  & 117  & 86  & 4686  & 310  \\
\textbf{Small test} &  & 142  & 82  & 4683  & 307  \\
\textbf{Large test} &  & 285  & 82  & 4683  & 307  \\ \bottomrule
\end{tabular}
\caption{Number of graphs in the train, validation and test sets.}
\label{app:tb_stat}

\end{table}

\begin{table}[ht]
\centering
\begin{tabular}{l|ccc}
\toprule
\textbf{Datasets} & \multicolumn{3}{c}{\textbf{Graph size}} \\ \midrule
\textbf{BBBP}              & \textbf{Train}    &  \textbf{Small test}    &  \textbf{Large test}   \\
\textbf{Max}               & 27       & 27            & 132          \\
\textbf{Mean}              & 19.3     & 19            & 30.5         \\
\midrule

\textbf{PROTEINS }         &          &               &              \\
\textbf{Max}               & 26       & 26            & 620          \\
\textbf{Mean}              & 14.9     & 16.8          & 132.8        \\
\midrule

\textbf{GraphSST2}         &          &               &              \\
\textbf{Max}               & 7        & 7             & 56           \\
\textbf{Mean}              & 4.1      & 4.1           & 32           \\
\midrule

\textbf{NCI1}              &          &               &              \\
\textbf{Max}               & 27       & 27            & 111          \\
\textbf{Mean}              & 20.4     & 20.6          & 58.4         \\ \bottomrule
\end{tabular}
\caption{Statistics on graph sizes in the train, small and large test sets.}
\label{app:tb_stat_size}

\end{table}

\section{Data Pre-processing}
\label{app:data pre-processing}
This section outlines the data pre-processing techniques employed to prepare the datasets for analysis.

\textbf{Data splits. } For each dataset, we create four distinct sets: training, validation, small test, and large test. The large test set contains graphs that are significantly larger than those in the other sets. To generate the training, validation, and small test sets, we first select the smallest 50\% of graphs from each dataset. These graphs are then randomly split in a 70:15:15 ratio for the training, validation, and small test sets, respectively. It is important to note that this split is performed within each class to maintain a consistent label distribution across the training, validation, and small test sets. The large test set is formed by selecting graphs from the remaining pool, ensuring an equal number of graphs per class as in the small test set. This selection process starts from the largest graph in each class, aiming to match the class distribution observed in the small test subset.

\textbf{Upsampling. } Despite careful data splitting, the BBBP dataset presents a significant class imbalance. To mitigate the risk of training a model biased towards the majority class, we employ an upsampling strategy during the training phase. Specifically, graphs belonging to class 0 in the BBBP dataset are upsampled at a 6:1 ratio. %The resulting dataset statistics on the number of graphs are detailed in Table \ref{app:tb_stat} 
{The detailed statistics of the resulting dataset are presented in Tables~\ref{app:tb_stat} and \ref{app:tb_stat_size}.}

% \begin{table}[]
% \centering
% \begin{tabular}{llcccc}
% \hline
% \textbf{Dataset} &
%    &
%   \multicolumn{1}{l}{\textbf{BBBP}} &
%   \multicolumn{1}{l}{\textbf{PROTEINS}} &
%   \multicolumn{1}{l}{\textbf{GraphSST2}} &
%   \multicolumn{1}{l}{\textbf{NCI1}} \\ \hline
% Train      &  & 1026 & 389 & 21857 & 1438 \\
% Validation &  & 117  & 86  & 4686  & 310  \\
% Small test &  & 142  & 82  & 4683  & 307  \\
% Large test &  & 285  & 82  & 4683  & 307  \\ \hline
% \end{tabular}
% \label{app:tb_stat}
% \caption{Dataset statistics on the number of graphs}
% \end{table}

\section{Training Details}
\label{app:Training Details}

% This section delineates the training process details to facilitate reproducibility and ensure transparent comparisons. 
{This section outlines the details of the training process and the choices of hyperparameters}. We first offer detailed descriptions of the baseline models employed in our study. 

\begin{itemize}
    \item \textbf{SizeShiftReg} \cite{buffelli2022sizeshiftreg}: This method introduces a regularization approach using graph coarsening techniques to simulate size variations within the training set, thereby enhancing the size generalizability of GNNs in the graph classification tasks. The coarsening technique produces a simplified version of the original graph, preserving certain properties while altering its size, to better accommodate size shifts during training. \\

    \item \textbf{CIGAV2} \cite{chen2022learning}: This framework ensures out-of-distribution generalization under various distribution shifts by capturing the invariance of graphs. Specifically, it employs Structural Causal Models (SCMs) to characterize these distribution shifts in graphs and posits that GNNs remain invariant to these shifts if they focus on invariant and critical subgraphs. \\

    \item \textbf{RPGNN} \cite{murphy2019relational}: This framework has good representational power and is invariant to graph isomorphism. Built on the principles of finite partial exchangeability, it is model-agnostic and theoretically grounded.
\\
    
    \item \textbf{IRM} \cite{arjovsky2019invariant}: This learning paradigm aims to distinguish between properties of the training data that indicate spurious correlations and those that represent the actual phenomenon of interest, thereby enhancing generalization.

\end{itemize}

\begin{table}[ht]
{
\centering
\begin{tabular}{l|ll|cc|cc}
\toprule
\multirow{ 2}{*}{\diagbox{\textbf{Dataset}}{\textbf{Models}}} &
  \multicolumn{2}{c|}{\textbf{GCN}} &
  \multicolumn{2}{c|}{\textbf{GIN}} &
  \multicolumn{2}{c}{\textbf{GT}} \\ \cline{2-7}
&
  $\beta_1$ &
  $\beta_3$ &
  \multicolumn{1}{l}{$\beta_1$} &
  \multicolumn{1}{l|}{$\beta_3$} &
  \multicolumn{1}{l}{$\beta_1$} &
  \multicolumn{1}{l}{$\beta_3$} \\ \midrule
\textbf{BBBP}      & 0.5  & 5e4 & 0.05 & 5e8 & 0.1 & 5e9 \\
\textbf{PROTEINS}  & 0.05 & 5e4 & 0.05 & 5e4 & 0.1 & 1e4 \\
\textbf{GraphSST2} & 0.5  & 5e4 & 0.15 & 5e8 & 0.1 & 1e8 \\
\textbf{NCI1}      & 0.5  & 5e4 & 0.05 & 10  & 0.1 & 5e9 \\ \bottomrule
\end{tabular}
\caption{Setting of $\beta_1$ and $\beta_3$.}
\label{tb_hyper_model}
}
\end{table}

\begin{table}[ht]
\centering
\begin{tabular}{l|ll|lll}
\toprule
\multirow{2}{*}{\diagbox{\textbf{HPs}}{\textbf{Dataset}}}&
\multicolumn{2}{c}{\textbf{BBBP}} & \textbf{} & \multicolumn{2}{c}{\textbf{PROTEINS}} \\ \cline{2-6}
& Small                         & Large               &                & Small                         & Large               \\ \midrule
$\beta_1$         & $\beta_2 = 1$, $\beta_3 = 5e4$   &                     & $\beta_1$        & $\beta_2 = 1$, $\beta_3 = 5e4$    &                     \\ \hline
0.3             & 92.38 $_{\pm 0.31}$          & 82.91 $_{\pm 0.05}$ & 0.03           & 76.35 $_{\pm 0.02}$           & 81.92 $_{\pm 0.01}$ \\
0.4             & 89.74 $_{\pm 0.01}$          & 81.40 $_{\pm 0.02}$ & 0.04           & 77.42 $_{\pm 0.01}$           & 81.11 $_{\pm 0.04}$ \\
0.5             & 92.92 $_{\pm 0.09}$          & 85.63 $_{\pm 0.07}$ & 0.05           & 75.0 $_{\pm 0.21}$            & 83.41 $_{\pm 0.31}$ \\
0.6             & 89.43 $_{\pm 0.01}$          & 82.71 $_{\pm 0.05}$ & 0.06           & 76.09 $_{\pm 0.01}$           & 82.57 $_{\pm 0.44}$ \\
0.7             & 90.78 $_{\pm 0.01}$          & 82.18 $_{\pm 0.05}$ & 0.07           & 76.23 $_{\pm 0.01}$           & 81.40 $_{\pm 0.40}$ \\ \midrule
$\beta_2$         & $\beta_1 = 0.5$, $\beta_3 = 5e4$ &                     & $\beta_2$        & $\beta_1 = 0.05$, $\beta_3 = 5e4$ &                     \\ \hline
0.8             & 90.44 $_{\pm 0.02}$           & 82.13 $_{\pm 0.08}$ & 0.8            & 75.37 $_{\pm 0.03}$           & 80.11 $_{\pm 0.20}$ \\
0.9             & 91.43 $_{\pm 0.01}$           & 80.83 $_{\pm 0.19}$ & 0.9            & 76.4 $_{\pm 0.20}$            & 84.09 $_{\pm 0.07}$ \\
1               & 92.92 $_{\pm 0.09}$           & 85.63 $_{\pm 0.07}$ & 1              & 75.0 $_{\pm 0.21}$            & 83.41 $_{\pm 0.31}$ \\
1.1             & 92.97 $_{\pm 0.01}$           & 84.11 $_{\pm 0.12}$ & 1.1            & 77.15 $_{\pm 0.01}$           & 82.44 $_{\pm 0.25}$ \\
1.2             & 92.66 $_{\pm 0.01}$           & 81.12 $_{\pm 0.01}$ & 1.2            & 76.09 $_{\pm 0.02}$           & 81.33 $_{\pm 0.04}$ \\ \midrule
$\beta_3 (*1e4)$ & $\beta_1 = 0.5$, $\beta_2 = 1$    &                     & $\beta_3 (*1e4)$ & $\beta_2 = 1$, $\beta_1 = 0.05$   &                     \\ \hline
3               & 89.31 $_{\pm 0.11}$           & 81.76 $_{\pm 0.01}$ & 3              & 76.11 $_{\pm 0.01}$           & 80.90 $_{\pm 0.02}$ \\
4               & 92.00 $_{\pm 0.35}$           & 84.02 $_{\pm 0.05}$ & 4              & 77.90 $_{\pm 0.02}$           & 81.56 $_{\pm 0.09}$ \\
5               & 92.92 $_{\pm 0.09}$           & 85.63 $_{\pm 0.07}$ & 5              & 75.0 $_{\pm 0.21}$            & 83.41 $_{\pm 0.31}$ \\
6               & 92.66 $_{\pm 0.09}$           & 84.87 $_{\pm 0.09}$ & 6              & 75.75 $_{\pm 0.01}$           & 80.91 $_{\pm 0.51}$ \\
7               & 91.87 $_{\pm 0.03}$           & 83.58 $_{\pm 0.02}$ & 7              & 76.41 $_{\pm 0.02}$           & 80.86 $_{\pm 0.25}$ \\ \bottomrule
\end{tabular}
\caption{Impact of hyperparameter choices on graph classification performance across small and large test graphs in BBBP and PROTEINS datasets. Performance is assessed using average $F_1$ scores and their standard deviations.}
\label{tb_hyperpara}
\end{table}

% For uniformity across experiments and to foster fair comparisons, we use the same hyperparameter across all experiments. 
Next, we describe our training setup and provide the configurations of our hyperparameters. We use a batch size of 32 and the Adam optimizer with a learning rate of 0.001.
% omitting gradient clipping. 
For each of the GNN backbones, we employ three convolutional layers and a 
% GCN backbone model, three graph isomorphism network (GIN) layers as GIN backbone and three graph transformer (GT) operators as GT backbone. 
global mean pooling layer to generate the graph representations. To address overfitting, we implement an early stopping mechanism. This mechanism operates with a patience interval of 50 epochs, and the selection criterion is the lowest validation loss. We run each setup five times and report the average $F_1$ scores and standard deviations in Table \ref{tb:abl}.

The hyperparameter settings for the baseline models are consistent with those specified in their original publications to ensure fidelity. Our hyperparameters are chosen based on the performance on the validation set, which contains graphs of similar sizes to the training data.
For \method, we set $k_1 = k_2 = 0.2N$ in the augmentation process, where $N$ is the number of nodes for an input graph. We set $\beta_2 = 1$ for all the experiments. The values for $\beta_1$ and $\beta_3$ are specified in Table \ref{tb_hyper_model}.

%For GCN, we set $\beta_1 = 0.05$ for PROTEINS, and set $\beta_1 = 0.5,  \beta_3=5e4$ for the rest experiments; for GIN, we set $\beta_1 = 0.15$ for GraphSST2, and set $\beta_1 = 0.05$ for the rest datasets. We further set $\beta_3 = 5e4$ for PROTEINS, $\beta_3 = 10$ for NCI1 and $\beta_3 = 5e8$ for all other datasets. For GT, we set $\beta_1 = 0.1$ for all datasets, $\beta_3 = 5e9, 1e4, 1e8, 5e9$ for BBBP, PROTEINS, NCI1 and GraphSST2, respectively. 

% The previously described upsampling technique is applied to all GNN backbones. 

 In Table \ref{tb_hyperpara}, we present our model's performance with various hyperparameter settings. We vary one hyperparameter at a time while keeping the others constant. For these experiments, we use the GCN~\cite{kipf2016semi} backbone. Our results indicate that our method is insensitive to hyperparameter changes and demonstrates consistent performance.

\section{{Two Large Real-world Datasets}}
\label{Two more real-world datasets}

% \begin{table}[]
% \centering
% \begin{tabular}{l|cc|cc}
% \hline
% \textbf{\begin{tabular}[c]{@{}l@{}}Pretrained model \\ fed to PGExplainer\end{tabular}} & \multicolumn{2}{|c|}{\textbf{GIN}} & \multicolumn{2}{c}{\textbf{GT}} \\ \hline
% \textbf{Dataset}   & small              & large              & small              & large              \\ \hline
% \textbf{BBBP}      & 92.52 $_{\pm 0.08}$ & 84.87 $_{\pm 0.03}$ & 91.89 $_{\pm 0.08}$ & 85.08 $_{\pm 0.03}$ \\
% \textbf{PROTEINS}  & 75.00 $_{\pm 0.44}$ & 82.98 $_{\pm 0.08}$ & 74.91 $_{\pm 0.50}$ & 83.49 $_{\pm 0.02}$ \\
% \textbf{GraphSST2} & 89.99 $_{\pm 0.08}$ & 83.97 $_{\pm 0.21}$ & 89.33 $_{\pm 0.05}$ & 84.11 $_{\pm 0.22}$ \\
% \textbf{NCI1}      & 56.90 $_{\pm 0.05}$ & 41.26 $_{\pm 0.08}$ & 57.16 $_{\pm 0.11}$ & 42.07 $_{\pm 0.05}$ \\ \hline
% \end{tabular}
% \caption{Impact of different pretrained model fed to PGExplainer. Performance is assessed in average $F_1$ scores and their standard deviations.}
% \label{Pretrained model fed}
% \end{table}

We include two more real-world datasets, REDDIT-BINARY \cite{yanardag2015deep} and FRANKENSTEIN \cite{orsini2015graph}, and present their results in \Cref{tb:two more data}. Specifically, REDDIT-BINARY is a social network data. In each graph, nodes represent users, and there is an edge between them if at least one of them responds to the other’s comment. A graph is labeled based on whether it belongs to a question/answer-based community or a discussion-based community. FRANKENSTEIN is a dataset of molecules, with each molecule represented as a graph. In this representation, vertices denote chemical atoms labeled with their respective symbols, while edges represent bond types. Tables \ref{tb:size} and \ref{tb:number graph stats} provide the statistics for the two datasets. For hyperparameters, we set $\beta_2=1$ for all the experiments.  The values for $\beta_1$ and $\beta_3$ are specified in Table \ref{hp_reddit_frank}. According to these tables, the average size of test graphs exceeds 1300 nodes. The performance of \method and the baselines in graph classification is outlined in Table \ref{tb:two more data}. These results confirm that our method successfully generalizes to graphs that are up to ten times larger than those used in training.

\begin{table}[ht]
\centering
\resizebox{0.5\linewidth}{!}{

\begin{tabular}{l|ccc}
\toprule
\textbf{Datasets} & \multicolumn{3}{c}{\textbf{Graph size}} \\ \midrule
\textbf{REDDIT-BINARY}     & \textbf{Train}    & \textbf{Small test}     &  \textbf{Large test}   \\
\textbf{Max}               & 304      & 302           & 3782         \\
\textbf{Mean}              & 119.2    & 114           & 1319.5       \\
\midrule
\textbf{FRANKENSTEIN}      &          &               &              \\
\textbf{Max}               & 16       & 16            & 214          \\
\textbf{Mean}              & 10.5     & 10.7          & 40.4         \\ \bottomrule
\end{tabular}
}
\caption{Statistics on graph sizes in the train, small and large test sets.}
\label{tb:size}

\end{table}

\begin{table}[ht]
\centering
\resizebox{0.5\linewidth}{!}{

\begin{tabular}{lcc}
\toprule
\textbf{Dataset} & \textbf{REDDIT-BINARY} & \textbf{FRANKENSTEIN} \\ \midrule
\textbf{Train}            & 699                    & 1518                  \\
\textbf{Validation}       & 152                    & 327                   \\
\textbf{Small test }      & 149                    & 324                   \\
\textbf{Large test}       & 149                    & 324                   \\ \bottomrule
\end{tabular}
}
\caption{Number of graphs in the train, validation and test sets.}
\label{tb:number graph stats}

\end{table}

\begin{table}[ht]
\centering
\resizebox{0.5\linewidth}{!}{

\begin{tabular}{l|ll|cc|cc}
\toprule
\multirow{2}{*}{\diagbox{\textbf{Dataset}}{\textbf{Models}}}
&
  \multicolumn{2}{c|}{\textbf{GCN}} &
  \multicolumn{2}{c|}{\textbf{GIN}} &
  \multicolumn{2}{c}{\textbf{GT}} \\ \cline{2-7}
 &
  $\beta_1$ &
  $\beta_3$ &
  \multicolumn{1}{l}{$\beta_1$} &
  \multicolumn{1}{l|}{$\beta_3$} &
  \multicolumn{1}{l}{$\beta_1$} &
  \multicolumn{1}{l}{$\beta_3$} \\ \midrule
\textbf{REDDIT-BINARY} & 0.5 & 1e12 & 0.05 & 1e12 & 0.5 & 1e12 \\
\textbf{FRANKENSTEIN}  & 0.5 & 5e4  & 0.05 & 5e4  & 0.5 & 1e4 \\
\bottomrule
\end{tabular}
}
\caption{Setting of $\beta_1$ and $\beta_3$.}
\label{hp_reddit_frank}
\end{table}

\clearpage
\begin{table}[ht]
\centering
\resizebox{0.6\linewidth}{!}{

\begin{tabular}{l|cccc}
\toprule
\multirow{2}{*}{\diagbox{\textbf{Models}}{\textbf{Datasets}}} &
  \multicolumn{2}{c}{\textbf{REDDIT-BINARY}} &
  \multicolumn{2}{c}{\textbf{FRANKENSTEIN}} \\ 
  \cline{2-5}
 &
  \textbf{Small} &
  \textbf{Large} &
  \textbf{Small} &
  \textbf{Large} \\ \midrule
\textbf{GCN} &
  93.86 $_{\pm 0.01}$ &
  41.40 $_{\pm 0.02}$ &
  41.57 $_{\pm 0.01}$ &
  34.14 $_{\pm 0.02}$ \\
\textbf{GCN + IRM} &
  91.58 $_{\pm 0.01}$ &
  39.50 $_{\pm 0.04}$ &
  42.54 $_{\pm 0.02}$ &
  \textcolor{violet}{37.38 $_{\pm 0.02}$} \\
\textbf{GCN + SSR} &
  92.21 $_{\pm 0.02}$ &
  41.60 $_{\pm 0.02}$ &
  \textcolor{red}{45.47 $_{\pm 0.03}$} &
  37.00 $_{\pm 0.01}$ \\
\textbf{GCN + CIGAv2} &
  91.78 $_{\pm 0.01}$ &
  \textcolor{violet}{42.57 $_{\pm 0.02}$} &
  39.81 $_{\pm 0.01}$ &
  36.81 $_{\pm 0.01}$ \\
\textbf{GCN + RPGNN}&
  \textcolor{violet}{93.96 $_{\pm 0.01}$} &
  41.66 $_{\pm 0.04}$ &
  43.65 $_{\pm 0.01}$ &
  33.98 $_{\pm 0.01}$ \\
\textbf{GCN + \method} &
  \textcolor{red}{94.11 $_{\pm 0.03}$} &
  \textcolor{red}{45.16 $_{\pm 0.01}$} &
  \textcolor{violet}{44.80 $_{\pm 0.01}$} &
  \textcolor{red}{38.00 $_{\pm 0.02}$} \\ \midrule
\textbf{GIN} &
  88.47 $_{\pm 0.01}$ &
  37.44 $_{\pm 0.04}$ &
  54.69 $_{\pm 0.01}$ &
  37.38 $_{\pm 0.01}$ \\
\textbf{GIN + IRM} &
  88.70 $_{\pm 0.02}$ &
  41.36 $_{\pm 0.03}$ &
  57.33 $_{\pm 0.01}$ &
  43.64 $_{\pm 0.01}$ \\
\textbf{GIN + SSR} &
  \textcolor{violet}{89.10 $_{\pm 0.01}$} &
  \textcolor{violet}{46.76 $_{\pm 0.03}$} &
  \textcolor{red}{58.38 $_{\pm 0.01}$} &
  47.90 $_{\pm 0.01}$ \\
\textbf{GIN + CIGAv2} &
  \textcolor{red}{89.57 $_{\pm 0.03}$} &
  45.80 $_{\pm 0.03}$ &
  57.33 $_{\pm 0.02}$ &
  \textcolor{violet}{48.01 $_{\pm 0.02}$} \\
\textbf{GIN + RPGNN} &
  86.46 $_{\pm 0.01}$ &
  39.67 $_{\pm 0.04}$ &
  56.48 $_{\pm 0.01}$ &
  47.59 $_{\pm 0.03}$ \\
\textbf{GIN + \method} &
  88.91 $_{\pm 0.03}$ &
  \textcolor{red}{48.24 $_{\pm 0.05}$} &
  \textcolor{violet}{57.87 $_{\pm 0.01}$} &
  \textcolor{red}{48.30 $_{\pm 0.02}$} \\ \midrule
\textbf{GT} &
  \textcolor{red}{94.22 $_{\pm 0.01}$} &
  59.43 $_{\pm 0.08}$ &
  59.72 $_{\pm 0.01}$ &
  52.26 $_{\pm 0.01}$ \\
\textbf{GT + IRM} &
  92.15 $_{\pm 0.01}$ &
  58.41 $_{\pm 0.02}$ &
  59.70 $_{\pm 0.02}$ &
  51.65 $_{\pm 0.02}$ \\
\textbf{GT + SSR} &
  91.92 $_{\pm 0.01}$ &
  60.11 $_{\pm 0.02}$ &
  52.04 $_{\pm 0.04}$ &
  \textcolor{violet}{53.23 $_{\pm 0.03}$} \\
\textbf{GT + CIGAv2} &
  91.50 $_{\pm 0.01}$ &
  \textcolor{violet}{61.99 $_{\pm 0.01}$} &
  58.20 $_{\pm 0.01}$ &
  54.29 $_{\pm 0.02}$ \\
\textbf{GT + RPGNN} &
  92.40 $_{\pm 0.01}$ &
  59.22 $_{\pm 0.01}$ &
  \textcolor{red}{59.82 $_{\pm 0.02}$} &
  50.30 $_{\pm 0.01}$ \\
\textbf{GT + \method} &
  \textcolor{violet}{93.27 $_{\pm 0.03}$} &
  \textcolor{red}{63.19 $_{\pm 0.03}$} &
  \textcolor{violet}{59.76 $_{\pm 0.11}$} &
  \textcolor{red}{53.99 $_{\pm 0.01}$} \\ \bottomrule
\end{tabular}
}
\caption{Graph classification performance evaluated on two large real-world datasets. Results are reported in average $F_1$ scores along with their standard deviations. The highest $F_1$ scores are highlighted in red, and the second-highest are in violet. }

\label{tb:two more data}
\end{table}

% \section{\zh{Use of Pretrained GNN}}
% \label{Use of a Pretrained GNN}
% To minimize the impact of knowledge distillation from augmented views (i.e., task-invariant views) we utilize the very basic pre-trained GNN model, GCN \cite{kipf2016semi}. To further investigate the role of knowledge distillation in enhancing performance, we conduct experiments where we use PGExplainer \cite{luo2020parameterized} with other pre-trained models (GIN, and GraphTransformer) to generate task-invariant views. Table \ref{Pretrained model fed} shows the performance of our methods when using other pre-trained models.

% It can be seen from the results that (1) Our model has consistent performance with different pre-trained models for augmentation. (2) Using better pre-trained models, such as GT, the performance of our method is actually further improved compared to our results in the paper (Table \ref{tb:performance})

% Hence, the enhancements observed in our model do not stem from knowledge distillation of the pretrained models, as we utilized the least effective pretrained model. Rather, these improvements arise from the ability to discern and segregate size and task information.

\end{document}